\documentclass{article}

\usepackage{microtype}
\usepackage{graphicx}
\usepackage{subfigure}
\usepackage{booktabs}

\RequirePackage{algorithm}
\RequirePackage[noend]{algorithmic}
\RequirePackage{natbib}

\usepackage{amsmath}
\usepackage{amssymb}
\usepackage{mathtools}
\usepackage{amsthm}

\usepackage[utf8]{inputenc} 
\usepackage[T1]{fontenc}    
\usepackage{url}            
\usepackage{booktabs}       
\usepackage{amsfonts, dsfont}       
\usepackage{nicefrac}       
\usepackage{microtype}      
\usepackage{xcolor}        
\usepackage{enumitem}
\usepackage{soul}

\usepackage[colorlinks=true, citecolor = blue, pagebackref]{hyperref}

\usepackage{amsmath}
\usepackage{amssymb}
\usepackage{mathtools}
\usepackage{amsthm}
\usepackage{mathrsfs}

\theoremstyle{plain}

\newtheorem{theorem}{Theorem}
\newtheorem{proposition}[theorem]{Proposition}
\newtheorem{lemma}[theorem]{Lemma}

\newtheorem{definition}[theorem]{Definition}
\newtheorem{assumption}[theorem]{Assumption}

\newcommand{\hA}{\hat{A}}

\DeclareMathOperator*{\argmin}{arg\,min}

\newcommand{\PP}{\mathbb{P}}

\newcommand{\indi}{\mathds{1}}

\newcommand{\tA}{\tilde{A}}

\newcommand{\stat}{\mathscr{T}}
\newcommand{\decrule}{\mathscr{D}}
\newcommand{\algo}{\mathscr{A}}
\newcommand{\test}{(\algo, \tau, \decrule)}

\newcommand{\hfeas}{\mathcal{H}_{\mathsf{F}}}
\newcommand{\hinfeas}{\mathcal{H}_{\mathsf{I}}}
\newcommand{\instance}{(\mathcal{X}, A, \delta)}

\newcommand{\regret}{\mathscr{R}}
\newcommand{\hist}{\mathsf{H}}
\newcommand{\boundary}{\mathscr{B}}

\newcommand{\confset}{\mathscr{C}}
\newcommand{\decset}{\mathscr{D}}

\newcommand{\lil}{\mathrm{LIL}}

\newcommand{\tstat}{\widetilde\stat}
\newcommand{\ttau}{\widetilde\tau}
\newcommand{\qfeas}{\mathscr{Q}^{\mathsf{F}}}
\newcommand{\qinfeas}{\mathscr{Q}^{\mathsf{I}}}
\newcommand{\imin}{i_{\min}}
\newcommand{\tconfset}{\widetilde{\confset}}
\newcommand{\tearly}{\tau_{\mathrm{early}}}

\usepackage[textsize=tiny]{todonotes}

\usepackage{fullpage}

\begin{document}

\title{Testing the Feasibility of Linear Programs with Bandit Feedback}

\author{\begin{tabular}{c@{\qquad}c}
  Aditya Gangrade & Aditya Gopalan \\ Boston University, University of Michigan & Indian Institute of Science \\ \texttt{gangrade@bu.edu} & \texttt{aditya@iisc.ac.in} \\  & \\ Venkatesh Saligrama & Clayton Scott\\ Boston University & University of Michigan \\ \texttt{srv@bu.edu} & \texttt{clayscot@umich.edu}
\end{tabular}}

\date{\vspace{-\baselineskip}}

\maketitle

\begin{abstract}
While the recent literature has seen a surge in the study of constrained bandit problems, all existing methods for these begin by assuming the feasibility of the underlying problem. We initiate the study of testing such feasibility assumptions, and in particular address the problem in the linear bandit setting, thus characterising the costs of feasibility testing for an unknown linear program using bandit feedback. Concretely, we test if $\exists x: Ax \ge 0$ for an unknown $A \in \mathbb{R}^{m \times d}$, by playing a sequence of actions $x_t\in \mathbb{R}^d$, and observing $Ax_t + \mathrm{noise}$ in response. By identifying the hypothesis as determining the sign of the value of a minimax game, we construct a novel test based on low-regret algorithms and a nonasymptotic law of iterated logarithms. We prove that this test is reliable, and adapts to the `signal level,' $\Gamma,$ of any instance, with mean sample costs scaling as $\widetilde{O}(d^2/\Gamma^2)$. We complement this by a minimax lower bound of $\Omega(d/\Gamma^2)$ for sample costs of reliable tests, dominating prior asymptotic lower bounds by capturing the dependence on $d$, and thus elucidating a basic insight missing in the extant literature on such problems. 
\end{abstract}

\section{Introduction} \label{sec:intro}

While the theory of single-objective bandit programs is well established, most practical situations of interest are multiobjective in character, e.g., clinicians trialling new treatments must balance the efficacy of the doses with the extent of their side-effects, and crowdsourcers must balance the speed of workers with the quality of their work. In cognisance of this basic fact, the recent literature has turned to the study of constrained bandit problems, wherein, along with rewards, one observes risk factors upon playing an action. For instance, along with treatment efficacy, one may measure kidney function scores using blood tests after a treatment. The goal becomes to maximise mean reward while ensuring that mean scores remain high \citep[e.g.][]{nathan2014diabetes}. 

Many methods have been proposed for such problems, both in settings where constraints are enforced in aggregate, or in each round (`safe bandits'), see \S\ref{sec:related_work}. However, every such method begins by assuming that the underlying program is feasible (or more; certain safe bandit methods require knowing a feasible ball). This is a significant assumption, since it amounts to saying that despite the fact that the risk factors are not well understood (hence the need for learning), it is known that the action space is well founded, and contains points that appropriately control the risk. This paper initiates the study of \emph{testing this assumption}. The result of such a test bears a strong utility towards such constrained settings: if negative, it would inform practitioners of the inadequacy of their design space, and spur necessary improvements, while if positive, it would yield a cheap certificate to justify searching for optimal solutions within the space. The main challenge lies in ensuring that the tests are reliable and sample-efficient (since if testing took as many samples as finding optima, the latter question would be moot). 

Concretely, we work in the linear bandit setting, i.e., in response to an action $x \in \mathcal{X}\subset \mathbb{R}^d$, we observe scores $S \in \mathbb{R}^m$ such that $\mathbb{E}[S|x] = Ax$, where $A$ is latent, and with the constraint structured as $Ax \ge \alpha$ for a given tolerance vector $\alpha$. We study the binary composite hypothesis testing problem of determining if there exists an $x: Ax \ge \alpha$ or not, with the goal of designing a sequential test that ensures that the probability of error is smaller than some given $\delta$. Such a test is carried out for some random time $\tau,$ corresponding directly to the sample costs, which we aim to minimise. Effectively we are testing if an unknown linear program (LP) is feasible, and we may equivalently phrase the problem as testing the sign of the minimax value $\Gamma := \max_{x \in \mathcal{X}} \min_{i} (Ax - \alpha)^i$. Also note that by incorporating the objective as a constraint vector, and a proposed optimal value as a constraint level, this test also corresponds to solving the recognition (or decision) version of the underlying LP \citep[e.g.,][Ch.~15]{papadimitriou1998combinatorial}.

This problem falls within the broad purview of pure exploration bandit problems, and specifically the so-called \emph{minimum threshold problem}, which has been studied in the \emph{multi-armed case} for a \emph{single constraint} \citep[e.g.][also see \S\ref{sec:related_work}]{kaufmann2018sequential}. Most of this literature focuses on the asymptotic setting of $\delta \searrow 0$, and the typical result is of the form \emph{if the instance is feasible, then there exist tests satisfying $\lim \smash{\frac{\Gamma^2\mathbb{E}[\tau]}{{2\log(1/\delta)}}} = 1$}. Prima facie this is good news, in that there is a well-developed body of methods with tight instance specific costs that do not depend on the dimension of the action set, $d$! However, this lack of dependence should give us pause, since it does not make sense: if, e.g., $\mathcal{X}$ were a simplex, and only one corner of it were feasible, then detecting this feasibility should require us to search along each of the axes of $\mathcal{X}$ to locate \emph{some} evidence, and so cost at least $\Omega(d)$ samples. The catch here lies in the limit, which implicitly enforces the regime $\delta = e^{-\omega(d)}$. Of course, even for modest $d$, such small a $\delta$ is practically irrelevant. Thus, even in the finite-armed case, the existing theory of feasibility testing does not offer a pertinent characterisation of the costs in scenarios of rich action spaces with rare informative actions.

\textbf{Our contributions} address this, and more. Concretely, we

\begin{itemize}[wide,nosep]
    \item Design novel and simple tests for feasibility based on exploiting low-regret methods and laws of iterated logarithm to certify the sign of the minimax value $\Gamma$.
    \item Analyse these tests, and show that they are reliable and well-adapted to $\Gamma,$ with stopping times scaling as $\widetilde{O}(d^2/\Gamma^2 + d\log(m/\delta)/\Gamma^2),$ thus demonstrating that the cost due to the number of constraints, $m$, is limited, and that testing is possible far more quickly than finding near-optimal points.
    
    \item Demonstrate a minimax lower bound of $\Omega(d/\Gamma^2)$ samples on the stopping time of reliable tests over feasible instances, thus showing that this uncaptured dependence is necessary.
\end{itemize}
We note that while the design approach of using low-regret methods for feasibility testing has appeared previously, their use arises either as subroutines in a complex method, or through modified versions of Thompson Sampling that are hard to even specify for the linear setting. Instead, our approach is directly motivated, and extremely simple, relying only on the standard technical tools of online linear regression and laws of iterated logarithms (LILs), employed in a new way to construct robust boundaries for our test statistics. Our results thus provide a new perspective on this testing problem, and more broadly on active hypothesis testing.

\subsection{Related Work}\label{sec:related_work}

\textbf{Minimum Threshold testing}. The single-objective finite-armed bandit setup \citep{lattimore2020bandit} posits $K< \infty$ actions, or `arms,' and in each round, a learner may `pull' one arm $k$ to obtain a signal with mean $a_k\in\mathbb{R}$. The minimum threshold testing problem is typically formulated in this setup, and demands testing if $\max_{k \in [1:K]} a_k \ge \alpha$ or $< \alpha$ (notice that this is our problem, but with $\mathcal{X}$ finite and mutually orthogonal, and $m = 1$; see \S\ref{appendix:finite_arm_problem}). The asymptotic behaviour of this problem has an asymmetric structure: if the instance is feasible, then lower bounds of the form $\smash{\liminf_{\delta \to 0} \log\frac{\mathbb{E}[\tau]}{\log(1/\delta)} \ge \frac{2}{\Gamma^2}}$ hold, while if the instance is infeasible, then the lower bound instead is $\smash{\sum_k \frac{2}{(\mu^k)^2}},$ since each arm must be shown to have negative mean. \citet{kaufmann2018sequential} proposed the problem, and a `hyper-optimistic' version of Thompson Sampling (TS) for it, called Murphy Sampling (MS), which is TS but with priors supported only on the feasible instances, and rejection boundaries based on the GLRT. We note that the resulting stopping times were not analysed in this paper. \citet{degenne2019pure} proposed a version of track and stop for this problem, but only showed asymptotic upper bounds on stopping behaviour; subsequently with M\'{e}nard \citep{degenne2019non}, they proposed a complex approach based on a two player game, with one of the players taking actions over the set of probability distributions on all infeasible or all feasible instances. The resulting stopping time bounds are stated in terms of the regret of the above player, and explicit forms of these for moderate $\delta$ are not derived. Further work has continued to study the single objective, finite-armed setting as $\delta \searrow 0$: \citet{juneja2019sample} extend the problem to testing if the mean vector $(a^k)_{k \in [1:K]}$ lies in a given convex set, and propose a track-and-stop method; \citet{tabata2020bad} study index-based LUCB-type methods; \citet{qiao2023asymptotically} study testing if $0 \in (\min a_k, \max a_k)$, and propose a method that combines MS with two-arm sampling.\footnote{While \citet{qiao2023asymptotically} \emph{define} a very pertinent multiobjective problem, this is not analysed in their paper beyond an asymptotic lower bound that again does not capture $K$.}

Curiously, none of this work observes the simple fact that if only one arm were feasible, then searching for this arm must induce a $\Omega(K/\Gamma^2)$ sampling cost. This cost is significant when $1/\delta = \exp( o(K)),$ which is the practically relevant scenario of moderate $\delta$ and large $K$. In \S\ref{sec:lower_bound}, we show the the $\Omega(K/\Gamma^2)$ lower bound using the `simulator' technique of \citet{simchowitz2017simulator}. We note that while this method was previously applied to minimum threshold testing by \citet{kaufmann2018sequential}, they focused on generic bounds, and only recovered a $(\log(1/\delta)+1/K)\Gamma^{-2}$ lower bound. Instead, we show a minimax lower bound, losing this genericity, but capturing the linear dependence. 

Along with demonstrating the above fact, the key distinction of our work is that we study a \emph{multiobjective} feasibility problem in the more challenging (\S\ref{appendix:finite_arm_problem}) \emph{linear bandit setting}. We further note that many of the tests proposed for the finite-armed case are challenging to even define for the linear setting: MS requires sampling from the set of feasible instances $\{A \in \mathbb{R}^{m \times d} : \max_{\mathcal{X}}\min_i (Ax)^i \ge 0\}$, and the approach of \citet{degenne2019non} needs a low-regret algorithm for distributions over this highly nonconvex set. In sharp contrast, the tests we design are conceptually simple, and admit concrete bounds on sample costs. Thus, our work both extends this literature, and provides important basic insights for its nonasymptotic regime. It should be noted that one also expects statistical advantages: since the set of feasible instances is $md$ dimensional, regret bounds on the same would vary polynomially in $md$, and thus one should expect stopping times to scale at best polynomially in $md$ using the approach of \citet{degenne2019non}, while our method admits bounds scaling only as $\mathrm{poly}(d,\log m)$.

In passing, we mention the parallel problem of finding either \emph{all} feasible actions, called \textbf{thresholding bandits} \citep[e.g.][]{locatelli2016optimal}, and of finding \emph{one} feasible arm, called \textbf{good-arm identification} \citep[e.g.][]{kano2017good, jourdan2023anytime}, assuming that they exist. Lower bounds in this line of work also focus on the asymptotic regime for finite-armed single objective cases. Of course, these problems are clearly harder than our testing problem, and so our lower bound also have implications for them.

\textbf{Constrained and Safe Bandits.} Multiobjective problems in linear bandit settings, amounting to bandit linear programming, are formulated as either aggregate constraint satisfaction \citep[e.g.][]{badanidiyuru2013bandits,agrawal2014bandits,agrawal2016linear} or roundwise satisfaction \citep[called `safe bandits', e.g.][]{amani2019linear, katz2019top, moradipari2021safe, pacchiano2021stochastic, chen2022doubly, wang2022best, camilleri2022active}. All such work assumes the feasibility of the underlying linear program to start with, and certain approaches further require knowledge of a safe point in the interior of the feasible set. Our study is directly pertinent to safe linear bandits, and to aggregate constrained bandits if $\mathcal{X}$ is convex.

\textbf{Sequential Testing.} Finally, some of the technical motifs in our work have previously appeared in the sequential testing literature. Most pertinently, \citet{balsubramani2015sequential} define a test using the LIL, but without any actions (i.e., $|\mathcal{X}| = 1$). In their work, as in ours, the LIL is used to uniformly control the fluctuations of a noise process. 

\section{Definitions and Problem Statement}\label{sec:defi}

\emph{Notation.} For a matrix $M, M^i$ denotes the $i$th row of $M$, and for a vector $z, z^i$ is the $i$th component of $z$. For a positive semidefinite matrix $M,$ and a vector $z$, $\|z\|_M := \sqrt{z^\top M z}$ Standard Big-$O$ and Big-$\Omega$ are used, and $\widetilde{O}$ further hides polylogarithmic factors of the arguments: $f(u) = \widetilde{O}(g(u))$ if $\exists c : \limsup_{u \to \infty} \frac{f(u)}{g(u) \log^c g(u)} < \infty$.

\textbf{Setting.} An instance of a linear bandit feasibility testing problem is determined by a domain $\mathcal{X}$, a latent constraint matrix $A \in \mathbb{R}^{m \times d},$ and a error level $\delta \in (0,1),$ to test\footnote{notice that we have dropped the tolerance levels $\alpha$ in this definition. Since $\alpha$ is known a priori, this is without loss of generality: we can augment the dimension by appending a $1$ to each action, and $-\alpha^i$ to the $i$th row of the constraint matrix $A$.} \[ \hfeas : \exists x \in \mathcal{X} : Ax \ge 0 \quad \textrm{vs.}\quad \hinfeas : \forall x \in \mathcal{X} \exists i: (Ax)^i < 0,\] where $\hfeas$ should be read as the `feasibility hypothesis', and $\hinfeas$ as the `infeasibility hypothesis'. We shall also write $A \in \hfeas$ or $\in \hinfeas$ if the corresponding hypothesis is true.

\textbf{Information Acquisition} proceeds over rounds indexed by $t \in \mathbb{N}$. For each $t$, the tester selects some action $x_t$, and observes scores $S_t \in \mathbb{R}^m$ such that \( S_t = Ax_t + \zeta_t,\) where $\zeta_t$ is assumed to be a subGaussian noise process. The information set of the tester after acquiring feedback in round $t$ is $\hist_t := \{ (x_{s}, S_s)\}_{s \le t},$ and the choice $x_t$ must be adapted to the filtration generated by $\hist_{t-1}$. We let $X_{1:t} := \begin{bmatrix} x_1 & x_2 & \cdots & x_t\end{bmatrix}^\top, S_{1:t} := \begin{bmatrix} S_1 & S_2 & \cdots & S_t\end{bmatrix}$ denote the matrices whose rows are the $x$s and $S$s up to $t$.

A \textbf{Test} is comprised of three components: $\mathrm{(i)}$ a (possibly stochastic) \emph{action selecting algorithm} $\mathscr{A}: \mathcal{H}_{t-1} \to \mathcal{X},$ $\mathrm{(ii)}$ a \emph{stopping time} $\tau$ adapted to $\hist_{t}$,and  $\mathrm{(iii)}$ a \emph{decision rule} $\decrule: \hist_{\tau} \to \{\hfeas, \hinfeas\}.$ In each round, these are executed as follows: we begin by executing $\mathscr{A}$ to determine a new action for the round, and update the history with the feedback gained. We then check if $\tau = t$ to verify if we have accumulated enough information to reliably test, and if so, we stop, and if not, we conclude the round. Upon stopping, we evaluate the decision of $\decrule$, and return its output as the conclusion of the test. The design of $\test$ can of course depend on $(\mathcal{X}, \delta,m)$, but not on $A$. The basic reliability requirement for such a test is captured below.
\begin{definition}\label{def:reliability}
    A test $\test$ is said to be reliable if for any instance $\instance,$ and $* \in \{\mathsf{F} , \mathsf{I} \}$ if  $A \in \mathcal{H}_*$, then it holds that $\mathbb{P}(\decrule(\hist_\tau) \neq \mathcal{H}_*) \le \delta$.  
\end{definition}

\begin{figure}[tb]
    \centering
    \includegraphics[width = 0.45\linewidth]{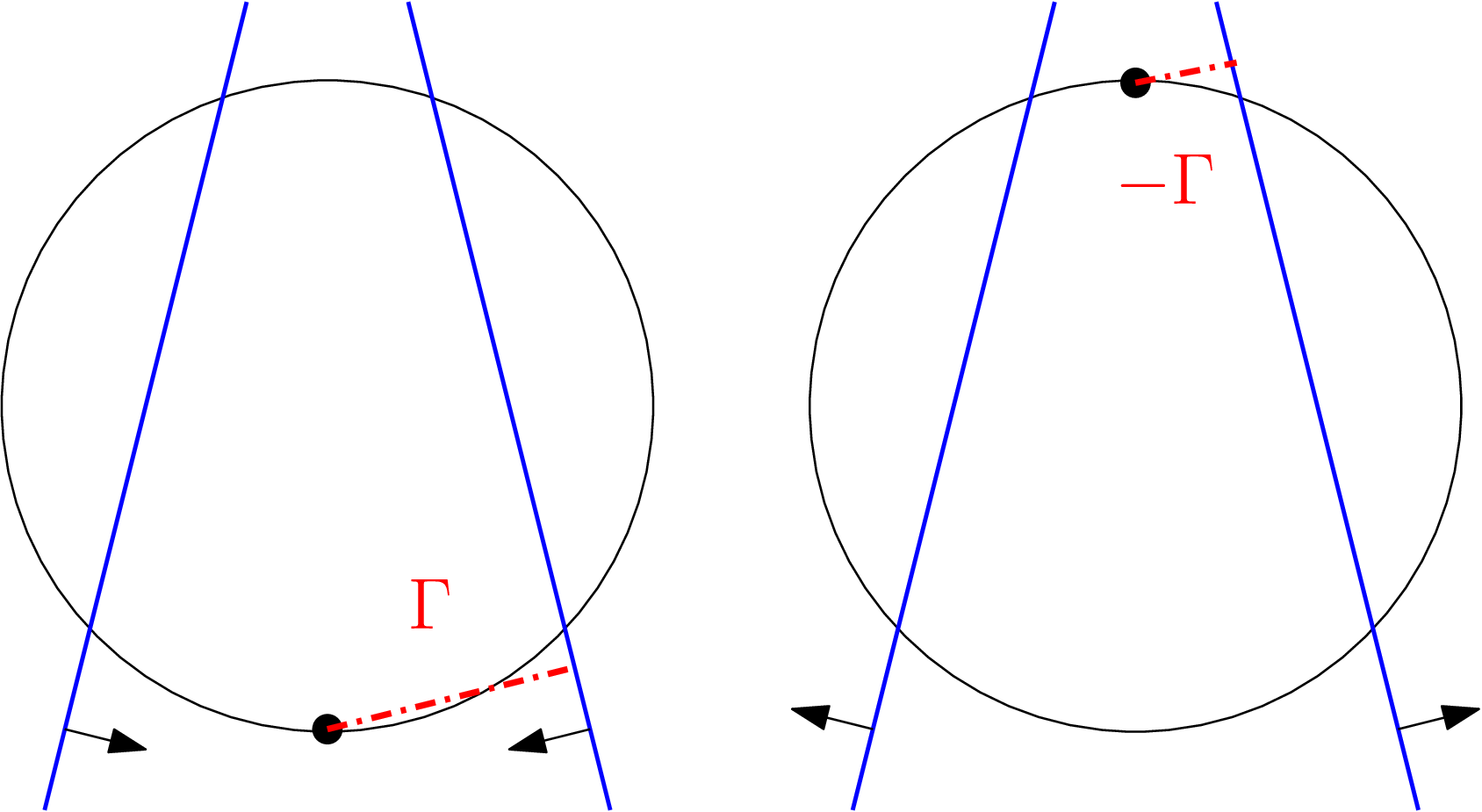}
    \caption{\footnotesize \emph{Illustration of the Signal Level}. The ball is $\mathcal{X},$ and lines with arrows indicate the feasible half spaces for each constraint, and assuming that $\|A^i\|=1$ for all $i$.  \emph{Left.} A feasible case; $\Gamma > 0$ is the distance of the marked point from the constraints, i.e., the length of the red dash-dotted line. \emph{Right.} An infeasible case with $-\Gamma > 0$ shown similarly.} 
    \label{fig:gamma_illustration}
\end{figure}
\textbf{Signal level, and adaptive timescale.} The hypotheses $\hfeas, \hinfeas$ can equivalently be defined according to the sign of $\max_{x} \min_{i \in [1:m]} (A x)^i.$ We define the \emph{signal level} of an instance as $\Gamma := \max_{x} \min_i (Ax)^i$. This is illustrated in Fig.~\ref{fig:gamma_illustration}. Notice that $|\Gamma|$ must enter the costs of testing. Indeed, even if we revealed to the tester the minimax $(x^*,i_*),$ and the value of $\Gamma$, since the KL divergence between $\mathcal{N}(\Gamma, 1)$ and $\mathcal{N}(-\Gamma,1)$ is $\Gamma^2,$ we would still $\Omega(\Gamma^{-2} \log(1/\delta))$ samples to determine the sign $(Ax^*)^{i_*}$ \citep[see, e.g.,][Ch.~13,14]{lattimore2020bandit}. Thus, $\Gamma^{-2}$ determines the minimal timescale for reliable testing, motivating
\begin{definition}\label{definition:validity_and_adaptedness}
    We say that a test is valid if it is reliable, and for any instance with signal level $\Gamma > 0,$ the test eventually stops, that is, $\PP(\tau < \infty) = 1.$ We further say that the test is well adapted to the signal level if it holds that for fixed $d,\delta,$ $\mathbb{E}[\tau] = O( \Gamma^{-2} \mathrm{polylog}(\Gamma^{-2}))$.  
\end{definition}
Any well adapted and reliable test must be valid. Further, a well adapted test is \emph{fast} compared to finding near-optimal actions for safe bandit problems in feasible instances, since $\Gamma$ is determined by the `most-feasible' point in $\mathcal{X}$. For instance, consider a crowdsourcing scenario where we want to maximise the net amount of work done in a given time period, subject to meeting a quality score constraint of $Q$ units. Since the number of very high quality workers in the pool may be limited, optimal solutions would need to use relatively low quality workers. However, verifying that such workers meet the constraint requires time proportional to $\min_w (Q^w-Q)^{-2},$ where $Q^w$ is the mean quality of worker $w$. In contrast, $\Gamma$ is determined by $\max_w (Q^w - Q)$, i.e., how good the best workers are, and so $\Gamma^{-2}$ is much smaller than the time scale required to find an optimal solution.

\textbf{Standard Conditions.} While briefly discussed above, we explicitly impose the following conditions, standard in the linear bandit literature \citep[see, e.g.,][]{abbasi2011improved}. All results in this paper assume the following. \begin{assumption}
    We assume that the instance is bounded,\footnote{If we are augmenting the dimension to account for nonzero $\alpha,$ these conditions apply only to the unaugmented $A,x$.} that is, $\mathcal{X} \subset \{\|x\| \le 1\},$ and $\{\forall i, \|A^i\| \le 1\}.$ We also assume the noise $\zeta_t$ to be conditionally $1$-subGaussian, i.e., \[\mathbb{E}[\zeta_t|\mathscr{G}_t] = 0, \forall \lambda \in \mathbb{R}^m, \mathbb{E}[\exp( \lambda^\top  \zeta_t) |\mathscr{G}_t] \le \exp(\|\lambda\|^2/2),\] where $\mathscr{G}_t$ is the filtration generated by $\hist_{t-1}, x_t,$ and any algorithmic randomness used by the test.
\end{assumption}
\section{Feasibility Tests Based on Low-Regret Methods}

We begin by heuristically motivating our test, and discussing the challenges arising in making this generic and formal. This is followed by an explicit description of the tests, along with main results analysing their performance.

\subsection{Motivation}\label{sec:motivation}

For simplicity, let us consider the case of $m = 1,$ so that $A = a^\top,$ for a vector $a$, and the signal level is $\Gamma = \max_{\mathcal{X}} a^\top x$. Due to the duality between testing and confidence sets \citep[][\S3.5]{lehmann2005testing}, a principled approach to testing the sign of $\Gamma$ is to build a confidence sequence for it, i.e., processes $\ell_t \le u_t$ such that with high probability, $\forall t, \Gamma \in (\ell_t,u_t)$. We naturally stop when $\ell_t u_t > 0,$ and decide on a hypothesis using the sign of $\ell_t$ on stopping. Any such confidence set in turn builds an estimate of $\Gamma$ itself, that is, some statistic that eventually converges to $\Gamma,$ at least if we did not stop. This raises the following basic question: how can we estimate $\max a^\top x$ without knowing where the maximum lies? A simple resolution to this comes from using low-regret methods for linear bandits. 

The linear bandit problem is parameterised by an objective $\theta,$ and a domain $\mathcal{X}$, and a method for it picks actions $x_t$ sequentially with the aim to minimise the \emph{pseudoregret} $\regret_t := \sum \max_x \theta^\top x - \theta^\top x_t$, using feedback of the form $\theta^\top x_t + \mathrm{noise}$. For `good' algorithms, $\regret_t$ scales as $\widetilde{O}(\sqrt{d^2t}),$ at least in expectation \citep[e.g.][Ch.19]{lattimore2020bandit}. Now, notice that if we take $\mathscr{A}$ to be such an algorithm executed with the feedback $S_t = a^\top x_t + \zeta_t,$ then the statistic $\stat_t/t$, where $\stat_t := \sum S_s$, should eventually converge to $\max_{x} a^\top x = \Gamma$. Indeed \[ \stat_t = \sum_{ s\le t} S_s = \sum_{s \le t} a^\top x_s + \sum_{s \le t} \zeta_s,\] and so the error in this estimate behaves as \[ \Gamma - \stat_t/t = \left(t\Gamma - \sum a^\top x_s \right)/t - \sum \zeta_s/t = (\regret_t + Z_t)/t, \] where $Z_t$ is a random walk, and so is typically $O(\sqrt{t})$. If $\regret_t \in [0, \widetilde{O}(\sqrt{d^2 t})],$ we can recover the \emph{sign} of $\Gamma$ reliably if \( t \Gamma \gg \regret_t + Z_t = \widetilde{O}(\sqrt{d^2 t}) \iff t \gg d^2/\Gamma^2.\) 

Formalising this heuristic approach, however, requires resolving two key issues. Firstly, we need to handle the multiobjective character of our testing problem: if $A \in \hinfeas$, there may be actions with only one out of $m$ constraints violated, and detecting this may be nontrivial. Secondly, to get a reliable test requires explicit statistics that can track the fluctuations in the noise, and in the pseudoregret (which is random due to the choice of $x_t$) in a reliable anytime way. These factors strongly influence the design of our tests. 

\subsection{Background on Online Linear Regression, and on Laws of Iterated Logarithms}

Before proceeding with describing our tests and results, we include a brief discussion of necessary background.

\textbf{Online Linear Regression.} We take the standard approach \citep{abbasi2011improved}. The {1-regularised} least squares (RLS) estimate of $A$ using $\hist_{t-1}$ is \begin{equation} \hat{A}_t := S_{1:t-1}X_{1:t-1}(X_{1:t-1}^\top X_{1:t-1} + I)^{-1}.\end{equation} Let us define the signal strength as $V_t := \sum_{s < t} x_s x_s^\top + I,$ and for $\delta \in (0,1)$, the $m$-confidence radius as \[ \omega_t(\delta) = 1 + \sqrt{\frac12 \log \frac{m \sqrt{\det V_t}}{\delta}}. \] The main results are based on the following two concepts, which we explicitly delineate. 

\begin{definition}
    For any time $t$, the \emph{RLS confidence set} is \[ \confset_t(\delta) := \{ \tA : \forall \textrm{ rows } i, \|\tA^i - \hat{A}_t^i\|_{V_t} \le \omega_t(\delta)\},\] and the \emph{local noise-scale} is \(\rho_t(x;\delta)  := 2\omega_t(\delta)\|x\|_{V_t^{-1}}.\)\vspace{-.5\baselineskip}
\end{definition}
Evidently, the set $\confset_t$ captures the $\tA$ that are plausible values of $A$ given $\hist_{t-1},$ the information available at the start of round $t$. We shall use the following standard results on the consistency of $\confset_t$ \citep{abbasi2011improved}. \begin{lemma} \label{lemma:online_linear_regression}
    For any instance and sequence of actions $\{x_t\},$ \[ \mathbb{P}(\exists t: A \not\in \confset_t(\delta)) \le \delta. \] Further, if $A \in \confset_t(\delta),$ then \[ \forall  \tA \in \confset_t(\delta),x\in \mathcal{X}: |\tA x - Ax| \le \rho_t(x;\delta) \mathbf{1}, \] where the inequality is interpreted row-wise. Finally, for any sequence of actions $\{x_t\}$, \[ \sum_{s \le t} \rho_t(x_t;\delta) \le \sqrt{6d t\omega_t(\delta) \log(1 + t/d)}. \]\vspace{-\baselineskip}\end{lemma}

\textbf{Nonasymptotic Law of Iterated Logarithms.} To the control the fluctuations introduced by the feedback noise, we use the following LIL due to \citet{howard2021time}. \begin{lemma} \label{lemma:LIL}
    For $t \in \mathbb{N}, \delta \in (0,1),$ let \[ \mathrm{LIL}(t,\delta) := \sqrt{4 t \log \frac{11\max(\log t,1)}{\delta}}.\] If $\eta_t \in \mathbb{R}$ is a conditionally centred and $1$-subGaussian sequence adapted to a filtration $\{\mathscr{G}_t\}$, then for $H_t := \sum \eta_t$, \[ \mathbb{P}(\exists t: |H_t| > \mathrm{LIL}(t,\delta)) \le \delta. \] \vspace{-\baselineskip}
\end{lemma}

\subsection{The Ellipsoidal Optimistic-Greedy Test}

We are now ready to describe our first proposed test, \textsc{eogt} which is specified in Algorithm~\ref{alg:eogt}. The test is parametrised by $\delta,$ and a constant $N$, and the algorithm proceeds by constructing a confidence set $\decset_t = \confset_t(\delta_t/2)$ for $A,$ which is the standard confidence set, but with a decaying confidence parameter $\delta_t = \delta t^{-N}$. It then selects both an action $x_t$, and a \emph{measured constraint} $i_t$ by solving the program\footnote{Note that the order of optimisation is important in (\ref{eqn:selection_algorithm_eogt}): since $(x, \tA) \mapsto \tA x$ is not quasiconvex, this value is in general not the same as $\smash{\min_i \max_{\tA,x} (\tA x)^i}$. Of course, it does hold that $\smash{\max_{\tA} \max_x \min_i (\tA x)^i = \max_{\tA} \min_i \max_x (\tA x)^i.}$} \begin{equation}
    \max_{\tA \in \decset_t} \max_{x \in \mathcal{X}} \min_{i \in [1:m]} (\tA x)^i. \label{eqn:selection_algorithm_eogt} 
\end{equation} The action $x_t$ is played, and the selected constraint  $i_t$ determines the main test statistic: \begin{equation} \stat_t := \sum_{s \le t} (S_s)^{i_s}.\end{equation} The test stops at $\tau := \inf\{t : |\stat_t| > \boundary_t(\delta)\},$ that is, when the magnitude of $\stat_t$ crosses the boundary \begin{equation} \boundary_t(\delta) := \sum_{s \le t} \rho_s(x_s;\delta_s/2) + \lil(t,\delta/2).\label{eqn:test_boundary} \end{equation}
This test can be interpreted in a game theoretic sense. Recall that $\Gamma$ is the value of the zero-sum game $\max_x \min_i (Ax)^i$. We can interpret the $\max$ player as a `feasibility-biased player', that moves first to pick an $x$ that makes $Ax$ large, and the $\min$ player as an `infeasibility-biased' player that counters with a constraint that $x$ does not meet well. 

\begin{algorithm}[t]
   \caption{Ellipsoidal Optimistic-Greedy Test (\textsc{eogt})}
   \label{alg:eogt}
\begin{algorithmic}[1]
   \STATE \textbf{Input}: $\delta \in (0,1), N \ge 2, \mathcal{X}, m.$ 
   \STATE \textbf{Initialise}: $\hist_0 \gets \varnothing, \stat_0 \gets 0, \boundary_0 \gets 0.$
   \FOR{$t = 1, 2, \dots$}
   \STATE $\delta_t \gets \delta t^{-N}, \decset_t \gets \confset_t(\delta_t/2).$ \hfill \emph{(Action Selection)}
   \STATE $(x_t, i_t) \gets \max_{\tA \in \decset_t, x \in \mathcal{X}} \min_i (\tA x)^i.$
   \STATE Play $x_t,$ and observe $S_t$. 
   \STATE Update $\hist_t \gets \hist_{t-1} \cup \{(x_t, S_t)\}.$ 
   \STATE Update $\stat_t \gets \sum_{s\le t}S_s^{i_s}, \boundary_t(\delta)$ as per (\ref{eqn:test_boundary})
   \IF{$|\stat_t| > \boundary_t(\delta)$} 
   \STATE STOP  \hfill\emph{(Stopping Rule)}
   \ENDIF
   \ENDFOR
   \STATE \textbf{Output} $\stat_t \overset{\hfeas}{\underset{\hinfeas}{\gtrless}} 0$ \hfill \emph{(Decision Rule)}
\end{algorithmic} 
\end{algorithm}

In \textsc{eogt}, action selection procedure is feasibility-biased: given the lack of knowledge of $A$, the feasibility player chooses a plausible $\tA$ that makes the value as high as possible, and the infeasibility player must abide by this choice of $\tA$. This is countered by the infeasibility-biased statistic $\stat_t$, in which only the infeasibility player's choice of $i_t$ is accounted for. This strikes a delicate balance: in the feasible case, as long as $x_t$ converges to a feasible subset of $\mathcal{X}$, $\stat_t$ eventually grows large and positive, while under infeasibility, if $i_t$ captures which constraints the $x_t$s consistently violate, $\stat_t$ eventually grows large and negative. Notice that while the feasibility player hedges their lack of information with optimism over the confidence ellipsoid, the infeasibility player acts greedily in the above test (and this structure inspires the name \textsc{eogt}). This greediness is natural if we view the infeasibility player as learner in a contextual stochastic full-feedback game, with context $(\tA_t, x_t)$, action $i_t$, and noisy feedback of the losses $\{(Ax_t)^{i}\}.$

The reliability of the test depends strongly on the form of the boundary $\boundary_t(\delta)$ above, which in turn arises from the analysis of the approach, which we shall now sketch.

\subsubsection{Analysis of Reliability}\label{sec:analysis_sketch} Naturally, the analysis differs if the problem is feasible or infeasible. Let us assume that $A \in \decset_t$ for all $t$. Since $\decset_t \subset \confset_t(\delta/2),$ this occurs with probability at least $1-\delta/2$.

\emph{Signal growth in the feasible case} relies on the optimism of the feasibility player. Let $(\tA_t, x_t, i_t)$ denote a solution to (\ref{eqn:selection_algorithm_eogt}). Since $A$ was feasible for this program, it must hold that $(\tA_t x_t)^{i_t} \ge \max_x \min_i (A x)^i = \Gamma$. Further, since $\tA \in \decset_t,$ using Lemma~\ref{lemma:online_linear_regression}, it holds that $(\tA_t x_t)^{i_t} \le (A x_t)^{i_t} + \rho_t(x_t;\delta_t/2),$ and so $(Ax_t)^{i_t} \ge \Gamma - \rho_t(x_t;\delta_t/2).$ Defining the noise process $Z_t = \sum_{s \le t} \zeta_s^{i_s}$ lets us conclude that \[ \stat_t \ge t \Gamma - \sum_{s \le t} \rho_s(x_s;\delta_s/2) + Z_t.\]

\emph{Signal growth in the infeasible case} instead relies on the extremisation in $i_t$ given $x_t$. Let $\imin(x) := \argmin_{i} (Ax)^i$. Since $i$ is the innermost optimised variable, and since $\imin(x_t)$ is feasible for the program (\ref{eqn:selection_algorithm_eogt}), it must hold that $(\tA_t x_t)^{i_t} \le (\tA_t x_t)^{\imin(x_t)}$. But, again, using Lemma~\ref{lemma:online_linear_regression}, $(\tA_t x_t)^{\imin(x_t)} \le (Ax_t)^{\imin(x_t)} + \rho_t(x_t;\delta_t/2),$ and further, $(Ax_t)^{\imin(x_t)} = \min_i (Ax_t)^i  \le \max_x \min_i (Ax)^i = \Gamma < 0.$ Therefore, in the infeasible case, \[ \stat_t \le t\Gamma + \sum \rho_s(x_s;\delta_s/2) + Z_t.\]

\emph{Boundary design and reliability.} Finally, the boundary design follows from control on the term $Z_t$ above. Notice that since $i_t$ is a predictable process, and $\zeta_t$ is conditionally $1$-subGaussian, it follows that $\eta_t := \zeta_t^{i_t}$ constitutes a centred, conditionally $1$-subGaussian process, and thus invoking the LIL (Lemma~\ref{lemma:LIL}) immediately yields \begin{lemma}\label{lemma:boundary_eogt}
    \textsc{eogt} ensures that, with probability at least $1-\delta,$ simultaneously for all $t \ge 1,$ \begin{align*}
    \textit{feasible case:}&\qquad \stat_t \ge t \Gamma - \boundary_t(\delta) \ge -\boundary_t(\delta),\\
    \textit{infeasible case:}&\qquad \stat_t \le -t|\Gamma| + \boundary_t(\delta) \le \boundary_t(\delta).
\end{align*}
\end{lemma} Since we stop when $|\stat_t| > \boundary_t(\delta),$ under the above event, upon stopping, $\stat_\tau \Gamma > 0,$ making the test reliable.

This leaves the question of the validity of the test, and the behaviour of $\mathbb{E}[\tau]$, which we now address.

\subsubsection{Control on Stopping time} Next, we describe our main result on the validity \textsc{eogt}, and the behaviour of $\mathbb{E}[\tau].$ To succinctly state this, we define \begin{align*}
    &T(\Gamma;\delta, N) := \inf\Big\{ t \ge 2d: t|\Gamma| > 2\lil(t,\delta/2) \\ &\quad + 4 d t^{1/2}\log(2t/d) + 2(dt\log(2t/d) \log\frac{2m}{\delta t^{-N}})^{\nicefrac{1}{2}} \Big\}
\end{align*}

Our main result, shown in \S\ref{appendix:eogt_analysis}, is \begin{theorem}\label{theorem:eogt_main_result}
    For any $\delta$ and $N > 1,$ the \textsc{eogt} is valid and well adapted. In particular, \[ \mathbb{E}[\tau] = O(  T(\Gamma/2;\delta,N)  + \delta/|\Gamma|).\]
\end{theorem} To interpret this result, in \S\ref{appendix:auxiliary_lemmata}, we employ worst-case bounds on $\sum_{s \le t} \rho_s(x_s;\delta_s)$ to control $T(\Gamma;\delta,N).$ \begin{lemma}\label{lemma:eogt_timescale}
    For any fixed $N$, $T(\Gamma;\delta,N)$ is bounded as \[O\left( \frac{d^2\log^2(d^2/\Gamma^2)}{\Gamma^2} + \frac{d \log(m/\delta) \log(d\log(m/\Gamma^2\delta))}{\Gamma^2} \right).\]
\end{lemma}
\noindent \textbf{Implications.} The main point that the above results make is that in the moderate $\delta$ regime of $\log 1/\delta = o(d),$ the typical stopping time of \textsc{eogt} is bounded as $d^2/\Gamma^2$ up to logarithmic factors. The factor of $d^2$ in this bound is deeply related to the analysis of online linear regression, and also commonly appears in the regret bounds (both in the worst case, $\sqrt{d^2 t},$ as well as in gapped instance-wise cases \citep{dani2008stochastic, abbasi2011improved}). 

Next, we note that the $d^2/\Gamma^2$ time-scale is typically much faster than that needed to approximately solve a feasible safe bandit instance: the best known method for finding a $\varepsilon$-optimal action for safe bandits requires $\Omega(d^2/\varepsilon^2)$ samples \citep{camilleri2022active}. However, as discussed after Definition~\ref{definition:validity_and_adaptedness}, $\Gamma$ is driven by the `safest' feasible action, while, since the optima lie at a constraint boundary, obtaining reasonably safe solutions requires setting $\varepsilon \ll \Gamma,$ making $d^2/\Gamma^2$ significantly smaller than $d^2/\varepsilon^2$. We also note that the above bound may be considerably outperformed by any run of the test: because $\boundary_t$ adapts to the trajectory, its growth can be much slower than the worst case bound that enters the definition of $T(\Gamma;\delta,N),$ allowing for fast stopping. 

Finally, observe that the dependence of this time scale on the number of constraints, $m$, is very mild, demonstrating that from a statistical point of view, many constraints are almost as easy to handle as one constraint. 

\subsection{Tail Behaviour, and the Tempered \textsc{eogt}}

While the expected stopping time of \textsc{eogt} is well behaved, its tail behaviour may be much poorer. Indeed, the best tail bound we could show, as detailed in \S\ref{appendix:eogt_tail}, is \begin{theorem}\label{theorem:eogt_tail}
    For every $\instance,$ and $\eta \in (0,\delta),$ \textsc{eogt} executed with parameters $(\delta,N)$ satisfies \begin{align*} &\PP(\tau > T(\Gamma;\delta,N)) \le \delta, \textit{and further,} \\ &\mathbb{P}\left(\tau > (2 + 1/|\Gamma|)\left\lceil ({\delta}/{\eta})^{1/N}\right\rceil + T(\Gamma/2;\eta,N)\right) \le \eta. \end{align*}
\end{theorem} Notice that the tail bound above is heavy, and the $\eta$-th quantile is only bounded as $O(1/|\Gamma| \eta^{-1/N}).$ It is likely that such behaviour is unavoidable due to \eqref{eqn:selection_algorithm_eogt}, due to which, if $m = 1,$ \textsc{eogt} directly exploits the OFUL algorithm of \citet{abbasi2011improved}, and the pseudoregret for this method is also heavy-tailed \citep{simchi2023regret}.

One way to avoid this poor behaviour is to instead select actions using variants of OFUL-type methods that achieve light-tailed pseudoregret. As summarised in Algorithm~\ref{alg:tempered}, we use the recently proposed approach of \citet{simchi2023regret} to construct such a test. The main difference is in selecting $(x_t, i_t)$ according to the program \begin{align} \max_{x \in \mathcal{X}} &\min_{i \in [1:m]} (\hat{A}_t x)^{i} + \mathrm{Rad}_t(x), \label{eqn:tempered_selection_rule} \\ \textit{where }  \mathrm{Rad_t}(x) &:=  (t/d)^{1/2}\|x\|_{V_t^{-1}}^2 +\sqrt{d \|x\|_{V_t^{-1}}^2}.  \notag\vspace{-\baselineskip} \end{align}

As a point of comparison, the selection rule (\ref{eqn:selection_algorithm_eogt}) can roughly be understood as (\ref{eqn:tempered_selection_rule}), but with $\smash{\mathrm{Rad}_t' = \sqrt{d \log t \|x\|_{V_t^{-1}}^2}}.$ Thus, the effect of $\mathrm{Rad}_t$ is to make the method more prone to exploration than (\ref{eqn:selection_algorithm_eogt}) if $t$ is large and $\smash{\|x\|_{V_t^{-1}}} \gg d/\sqrt{t}.$ So, the rule (\ref{eqn:tempered_selection_rule}) has the effect of tempering the tendency to exploitation of (\ref{eqn:selection_algorithm_eogt}), leading to the name `tempered \textsc{eogt}' (\textsc{t-eogt}). Importantly, observe that the selection rule (\ref{eqn:tempered_selection_rule}) makes no explicit reference to $\delta.$ 

The remaining algorithmic challenge is to define a boundary that can lead to a reliable test based on the above approach. In order to do this, we refine the techniques of \citet{simchi2023regret} to construct the following anytime tail bound, shown in \S\ref{appendix:tempered_eogt_tail}, for $\tstat_t$. We note that this also yields an \emph{anytime} tail bound for the regret of (\ref{eqn:tempered_selection_rule}) for linear bandits.
\begin{lemma}\label{lemma:tempered_boundary_derivation}
    For $\delta\in (0,1/2),$ let \begin{align*}
        \mathscr{Q}^{\mathsf{F}}_t(\delta) &:= 45\sqrt{dt \log^4 t}(d + \log(\nicefrac{8m}{\delta})) + \lil(t,\nicefrac{\delta}{2}), \\
        \mathscr{Q}^{\mathsf{I}}_t(\delta) &:= 27\sqrt{dt\log^3 t} ( \sqrt{d} + \log(\nicefrac{8m}{\delta})) + \lil(t,\nicefrac{\delta}{2}).
    \end{align*}
    Then, for $\widetilde{\stat}_t := \sum_{s \le t } S_s^{i_s}$ with actions picked via $(\ref{eqn:tempered_selection_rule})$, 
    \begin{alignat*}{3}
        \mathbb{P}(\forall t, \widetilde{\stat}_t \ge t \Gamma - \mathscr{Q}^{\mathsf{F}}_t(\delta) ) &\ge 1-\delta &&\quad \textit{(feasible case)}\\
        \mathbb{P}(\forall t, \widetilde{\stat}_t \le t\Gamma + \mathscr{Q}^{\mathsf{I}}_t(\delta)) &\ge 1-\delta && \quad \textit{(infeasible case)}
    \end{alignat*}
\end{lemma}

Naturally, we can reliably test via the stopping times \[ \widetilde{\tau} = \inf\{t : \widetilde\stat_t < -\mathscr{Q}^{\mathsf{F}}_t(\delta) \textrm{ or } \widetilde\stat_t > \mathscr{Q}^{\mathsf{I}}_t(\delta)\},\] deciding for $\hfeas$ if $\widetilde{\stat}_\tau > 0$. Using this, in \S\ref{appendix:tempered_eogt_proofs}, we show the following bounds along the lines of \S\ref{sec:analysis_sketch}.
\begin{theorem}\label{theorem:tempered_test}
    \textsc{t-eogt} is valid and well adapted, with \begin{equation*}
        \mathbb{E}[\widetilde{\tau}] = \widetilde{O}(d^3/\Gamma^2 + d/\Gamma^2 \log(8m/\delta)) 
    \end{equation*} where the $\widetilde{O}$ hides logarithmic dependence on $\nicefrac{d}{\Gamma^2},$ and $\log(\nicefrac m\delta)$. Further, there exists a $C$ scaling polylogarithmically in $\nicefrac{d}{\Gamma^2}$ and $\log(m/\eta)$ such that for all $ \eta \le \delta,$ \[ \mathbb{P}( \widetilde{\tau} \ge C d^3/\Gamma^2 + C d/\Gamma^2 \log(1/\eta) ) \le \eta.\]
\end{theorem} To contextualise the result, as well as this tempered test, let us consider the tradeoffs expressed in the above result. Compared to $\textsc{eogt},$ the procedure of $\textsc{t-eogt}$ suffers two main drawbacks: firstly, we see that the bound on the stopping time is significantly weaker, scaling as $d^3/\Gamma^2$ instead of $d^2/\Gamma^2,$ indicating a loss of performance. While this result may just be an artefact of the analysis, a more important drawback is that the test boundaries $\mathscr{Q}^{\mathsf{F}}, \mathscr{Q}^{\mathsf{I}}$ do not adapt to the sequence of actions actually played by the method, unlike $\boundary_t$, and instead are just deterministic processes that can be seen to essentially dominate $\sum \rho_s(x_s;\delta_s)$. Even if these bounds had tight constants (which they do not), such a nonadaptive stopping criterion cannot benefit from possible discovery of good actions early in the trajectory (accumulating on which would lead to contraction of $\rho_t,$ and thus decelaration of $\boundary_t$), and so cannot benefit from early termination that \textsc{eogt} may exploit in practice.

\begin{algorithm}[tb]
   \caption{Tempered \textsc{eogt} (\textsc{t-eogt})}
   \label{alg:tempered}
\begin{algorithmic}[1]
   \STATE \textbf{Input}: $\delta \in (0,1/2), \mathcal{X}, m.$ 
   \STATE \textbf{Initialise}: $\hist_0 \gets \varnothing, \tstat_0 \gets 0$
   \FOR{$t = 1, 2, \dots$}
   \STATE Compute $\hat{A}_t$.  \hfill \emph{(Arm Selection)}
   \STATE $(x_t, i_t) \gets \max_{x \in \mathcal{X}} \min_i (\hat{A}_t x)^i + \mathrm{Rad}_t(x).$
   \STATE Play $x_t,$ and observe $S_t$. 
   \STATE Update $\hist_t \gets \hist_{t-1} \cup \{(x_t, S_t)\}.$ 
   \STATE Update $\widetilde{\stat}_t \gets \sum_{s\le t}S_s^{i_s},$ and $\mathscr{Q}^{\mathsf{F}},\mathscr{Q}^{\mathsf{I}}$. 
   \IF{$\widetilde{\stat}_t > \mathscr{Q}^{\mathsf{F}}_t(\delta)$ or $\widetilde{\stat}_t < -\mathscr{Q}^{\mathsf{I}}_t(\delta)$} 
   \STATE STOP \hfill \emph{(Stopping Rule)}
   \ENDIF
   \ENDFOR
   \STATE \textbf{Output} $\widetilde\stat_t \overset{\hfeas}{\underset{\hinfeas}{\gtrless}} 0.$\hfill \emph{(Decision Rule)} 
\end{algorithmic}
\end{algorithm}

However, this weakness is balanced by considerably stronger tail behaviour: indeed, instead of the polynomial decay in tail probabilities for \textsc{eogt}, the above demonstrates exponential decay in the tails, with the decay scale further behaving as $d/\Gamma^2 \ll d^2/\Gamma^2$, meaning that typical fluctuations in the stopping time must be considerably smaller than the typical stopping time. The choice of test must depend the setting, and \textsc{t-eogt} should be preferred over \textsc{eogt} if rare but extreme testing delays yield strong penalties.

Finally, we would be remiss not to mention the curious difference in the boundaries $\mathscr{Q}^\mathsf{F}$ and $\mathscr{Q}^{\mathsf{I}},$ and in particular the weakness in $\mathscr{Q}^{\mathsf{F}}$ which is inherited in the bounds on $\mathbb{E}[\tau]$ in Theorem~\ref{theorem:tempered_test}. This difference arises because when controlling $\smash{\tstat_t}$ from below in the feasible case, we need the means $(Ax_t)^{i_t}$ to not be too far below the minimax value $\Gamma$, which is attained at some $x^* \neq x_t$. Ensuring this requires us to have control on both the noise scale at $x_t$ \emph{and that at} $x_*$. The latter is hard to accommodate in the analysis, which instead uses a lossy application of the AM-GM inequality to avoid it, but at the cost of the extra factor of $d^{1/2}$ in $\mathscr{Q}^{\mathsf{F}}$. On the other hand, when controlling $\tstat_t$ from above in the infeasible case, we only need to ensure that $i_t$ cannot do too poor a job of locating constraints that $x_t$ violates, which can be achieved by just considering the noise scale at $x_t$ itself. It may be possible to improve the analysis to reduce $\mathscr{Q}^{\mathsf{F}}$ down to $\mathscr{Q}^{\mathsf{I}}$, which we leave as a direction for future work.

\section{Minimax Lower Bounds}\label{sec:lower_bound}

We conclude the paper by discussing minimax lower bounds that capture the necessity of the dependence on $\Gamma^{-2},$ as well as at least a linear dependence on $d$ in generic bounds on stopping times for reliable tests. As we previously discussed in \S\ref{sec:intro} and \S\ref{sec:related_work}, the main point of comparison for these results are the corresponding \emph{instance-wise} lower bounds in the literature on the minimum threshold problem, which take essentially\footnote{the terms containing $\log(1/\delta)$ are always valid. The secondary terms behaving as $1/(K\Gamma^2)$ are upper bounds on the auxiliary terms appearing in the results of \citet{kaufmann2018sequential}.} the following form \citep{kaufmann2018sequential} \begin{alignat*}{3} \setlength{\abovedisplayskip}{.2\baselineskip}\setlength{\belowdisplayskip}{.2\baselineskip}
    \mathbb{E}[\widetilde{\tau}] &\ge 2\log(1/\delta)/\Gamma^2 + \nicefrac{1}{K\Gamma^2} \qquad &&\textit{(feasible case)},\\
        \mathbb{E}[\widetilde{\tau}] &\ge 2\log(1/\delta)\sum_k (\mu^k)^{-2} + \nicefrac{1}{K\Gamma^2} \qquad &&\textit{(infeasible case)}.
\end{alignat*} Notice that in the feasible case, the lower bound \emph{decays} with $K$. While the instance specific nature of the above bounds is desirable, we focus on minimax bounds capturing a linear dependence on $K$ (or, in our case, $d$) in specific instances. 

Our lower bound is based on a reduction to a finite action case, through the use of a simplex. The argument underlying this bound relies on the `simulator' technique of \citet{simchowitz2017simulator} for best arm identification (BAI). In fact, our main point, that the extant bounds for feasibility testing do not capture the dependence on $d$, is much the same as the observation of \citet{simchowitz2017simulator} that the analyses of `track-and-stop' BAI methods do not capture the right dependence on $K$ in BAI, again due to a focus on $\delta\to 0$. 

The construction underlying the bound is natural: we take $\mathcal{X}$ to be the simplex $\{x \ge 0: \sum x_i = 1\},$ and consider a single constraint matrix $a^\top$ for a vector $a \in [-1/2,1/2]^d$. The noise process is as follows: upon playing an action $x_t$, we sample $K_t \sim x_t$, and supply the tester with $\smash{a_{K_t} + \mathcal{N}(0,1/2)}.$ The vector $a$ is selected as a uniform permutation of the entries of $(\Gamma, -\varepsilon, -\varepsilon,\cdots, -\varepsilon)$, the intuition being that in order to detect the feasibility of such an instance, the test must sample the single `informative' extreme direction of the simplex at least $1/(\Gamma+\varepsilon)^2$ times. However, since this is selected uniformly at random, no method can generically identify this direction faster that just sampling uniformly, and so on average across the instances, $\tau =\Omega(d/\Gamma^2)$. Concretely, in \S\ref{appendix:lower_bound} we show the bound in a finite-armed case, and argue that the instance above must face the same costs. A technically interesting observation is that our argument relies on two uses of the simulator technique: we first compare the instance against an infeasible instance to argue that the arm with large signal must be played often, and we then use this result along with the simulator technique again to show that arms with poor signal must also be played often in an average sense across the permutations. Leaving the details to \S\ref{appendix:lower_bound}, this yields the following result.
\begin{theorem}\label{theorem:omega_d_lower_bound}
    For any $\Gamma,\delta \in (0,1/2)$ and any reliable $\test,$ there exists a \emph{feasible} instance $\instance$ with $m =1$ and signal level $\Gamma$ on which $\smash{\mathbb{E}[\tau] \ge  \frac{(1-2\delta)^3}{79}\cdot \frac{ d}{ \Gamma^2}}$. 
\end{theorem}

Note that utilising the existing results of \citet{kaufmann2018sequential} for the infeasible case, we can also recover a lower bound of $d/\Gamma^2 \log(1/\delta)$ if $|\Gamma| \le 1/\sqrt{d},$ by taking the instance $(-|\Gamma|, -|\Gamma|, \cdots, -|\Gamma|)$. Thus the linear dependence on $d$ is necessary over both feasible and infeasible cases.

We comment that the lower bound of $\Omega(d/\Gamma^2)$ remains far from the upper bounds of $O(d^2/\Gamma^2)$ in Theorem~\ref{theorem:eogt_main_result}. This linear in $d$ gap in the lower bound is a persistent occurrence in the theory of linear bandits, and shows up in any instance-specific control on the same, including in known regret lower bounds. As a result, resolving this is a task beyond the scope of the present paper. Nevertheless, our main point that the costs of testing depend strongly on $d$, unlike prior analysis suggsets, is well made by the above result.

\section{Simulations}

We conclude the paper by describing a heuristic implementation of \textsc{eogt}, and its behaviour on the simple case of testing the feasibility of two linear constraints over the unit ball.

\textbf{$L_1$ Confidence set.} Implementing \textsc{eogt} is challenging task, since the maximin program \eqref{eqn:selection_algorithm_eogt} is difficult to solve quickly. Indeed, even if $m = 1,$ i.e., there were only a single constraint, \eqref{eqn:selection_algorithm_eogt} requires us to implement the OFUL iteration, which is well known to be NP-hard due to the nonconvex objective $A^1 x$ \citep{dani2008stochastic}. 

To handle this, we begin with the standard relaxation used to implement OFUL, specifically by replacing the confidence ellipsoid $\confset_t(\delta)$ by the $L_1$-confidence set \[  \tconfset_t(\delta) := \{ \tA : \textrm{ for all rows } i, \| (\tA^i - \hat{A}_t^i)V_t^{1/2}\|_1 \le \sqrt{d}\omega\}.\] 

Since $\|\cdot\|_2 \le \|\cdot\|_1 \le \sqrt{d}\|\cdot\|_2,$  $\tconfset_t \supset \confset_t$, and thus $\tconfset_t$ is consistent w.h.p. Further, $\tconfset_t$ is in turn contained in a scaling of $\confset_t$ by a $\sqrt{d}$-factor, and thus the noise-scales over $\confset_t$ carries over, up to a loss of a $\sqrt{d}$ factor. This suggests that tests based on $\tconfset_t$ should use $\widetilde{O}(d^3/\Gamma^2)$ samples.

The main advantage, however, is that due to the $L_1$ structure, the set $\tconfset_t(\delta)$ only has $(2d)^m$ extreme points. This enables optimisation by a simple search over these extreme points, which at least for small $m$, leads to an implementable algorithm. In the following, we will only work with $m = 2$.

\textbf{Solving the Maximin Program.} Of course, even for a given $A$, $\max_x \min_i A^i x$ is nonobvious to solve since $i$ is discrete. We take the natural approach via convexifying:
\begin{align*} &\quad \max_{\tA \in \tconfset_t(\delta), x \in \cal X} \min_i \tA^i x  = \max_{\tA \in \tconfset_t(\delta)} \max_{x \in \cal X} \min_{\pi \in \Delta} \pi^\top \tA x,  \end{align*} where $\Delta$ is the simplex in $\mathbb{R}^m$. Now, for a fixed $\tA$, the maximin program over $(x,\pi)$ can be solved efficiently. The resulting $x,\tA$ can be used to directly minimise $(\tA x)^i$.

\textbf{Procedure.} Throughout the following, we will restrict attention to $\mathcal X = \{\|x\|_2 \le 1\}$. This enables a further simplification by using the minimax theorem for a fixed $\tA$: \begin{align*} \max_{x \in \cal X} \min_{\pi \in \Delta}\pi^\top \tA x  &= \min_{\pi \in \Delta} \max_{x \in \cal X} \pi^\top \tA x = \min_{\pi \in \Delta} \|\pi^\top \tA\|_2. \end{align*}

Overall, this yields the following procedure: we enumerate the extreme points of $\tconfset_t$, and for each, we solve for the minimising $\pi$ above, while keeping track of the maximum such value as we move over the extreme points. Upon conclusion, this yields a $\pi_t$ and a $\tA_t$ that solve the above. $x_t$ is then computed directly as $\pi_*^\top \tA_*/\|\pi_*^\top \tA_*\|.$ Given $x_t, \tA_t$, we finally direclty solve for $i_t$ by minimising $(\tA_t x_t)^i$.\footnote{For nonzero $\alpha,$ the objective is modified to $\|\pi^\top \tA\|_2 - \pi^\top \alpha$, and the final minimisation to discover $i_t$ then studies $(\tA_t x_t - \alpha)^{i}$.}

\textbf{Early Stopping for Feasible Instances.} Notice that in the feasible case, if we can ever argue that for some $x$, $\min_{\tconfset_t(\delta)} \min_i (\tA x)^i > 0$, then the test can already conclude. A natural candidate for such an $x$ is simply the running mean over the choices of $x_t$ played by $\textsc{eogt}$. The potential advantage of such a procedure is that it bypasses the possibly slow growth of $\stat_t$ when initial exploration chooses infeasible actions (which lead to a direct decrease in $\stat_t$, but do not affect the quality of the noise estimate at $x_t$ much). We also implement this early stopping procedure, and we will call the resulting stopping time $\tearly$.

\textbf{Settings} We study two scenarios: varying $d$ for a fixed $\Gamma,$ and varying $\Gamma$ for a fixed $d$. In each case we study both feasible and infeasible instances.

In the varying $d$ scenario, we pick the feasible instance $x_1 \ge 0, x_2 \ge 0$, and the infeasible instance $x_1 \ge 1/\sqrt{2}, x_1 \le -1/\sqrt{2}$. Notice that in either case, $\Gamma = 1/\sqrt{2}$. With these constraints, the simulation is run for $d \in [2:10]$. In the varying $\Gamma$ scenario, we fix $d = 4,$ and impose the constraints $x_1 \ge 1/\sqrt{2} - \Gamma, x_2 \ge 1/\sqrt{2} - \Gamma$ for the feasible setting, and the constraints $x_1 \ge \Gamma, x_1 \le - \Gamma$ in the infeasible case. The range $\Gamma \in [0.2,1]$ is studied at a grid of scale $0.1$.

Throughout, the feedback noise is independent Gaussian with standard deviation $\sigma = 0.1$ (the value of $\sigma$ is used in the confidence radii, and in general, $\tau$ should be proportional to $\sigma^2$). The parameter $\delta$ is set to $0.1$, $N = 1$, and all results are averaged over $50$ runs. The code was implemented in MATLAB, and executed on a consumer grade Ryzen 5 CPU, with no multithreading, and took about 4 hours to run.

\begin{figure}[t]
    \centering
    \includegraphics[width = .49\linewidth]{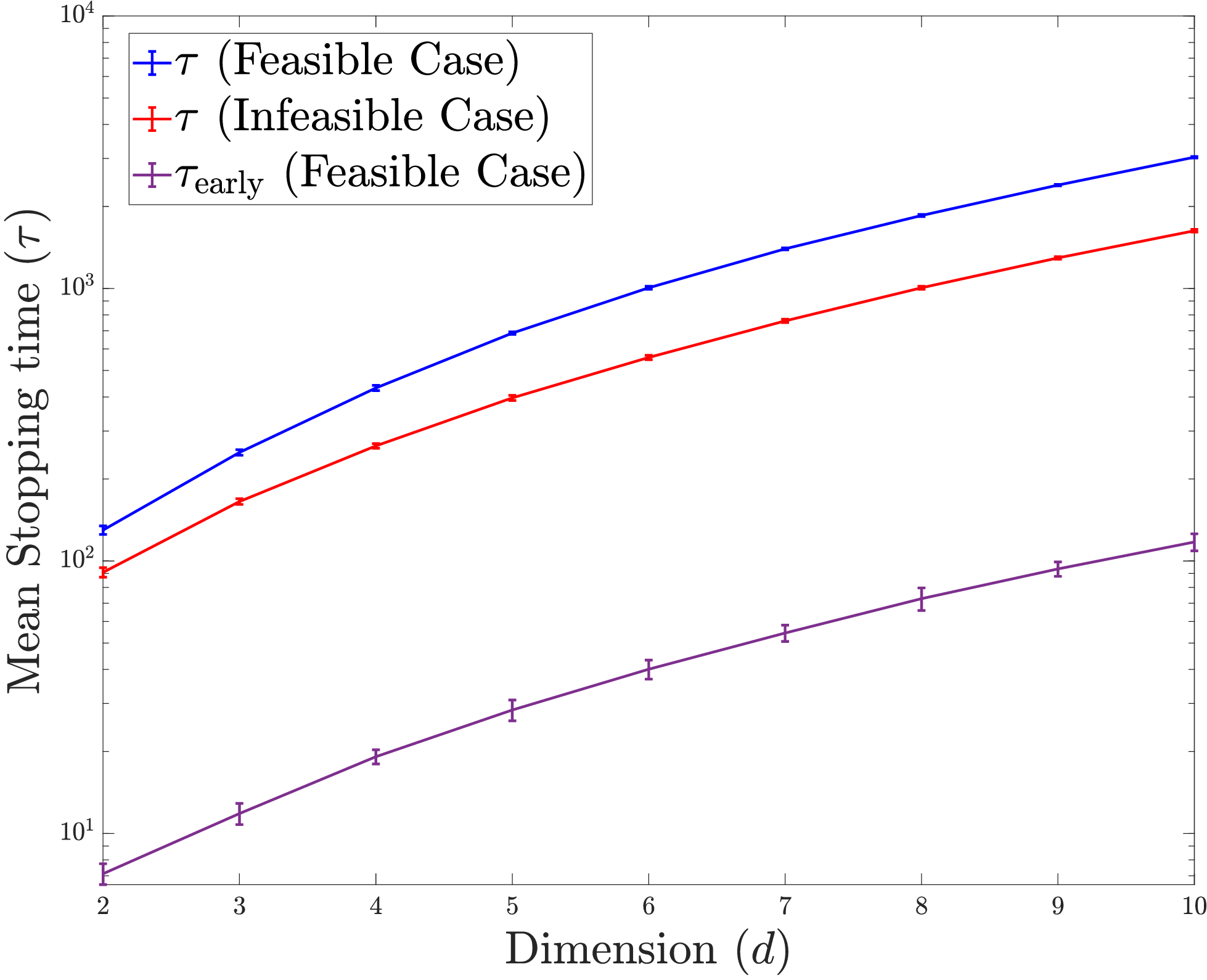}~
    \includegraphics[width = .49\linewidth]{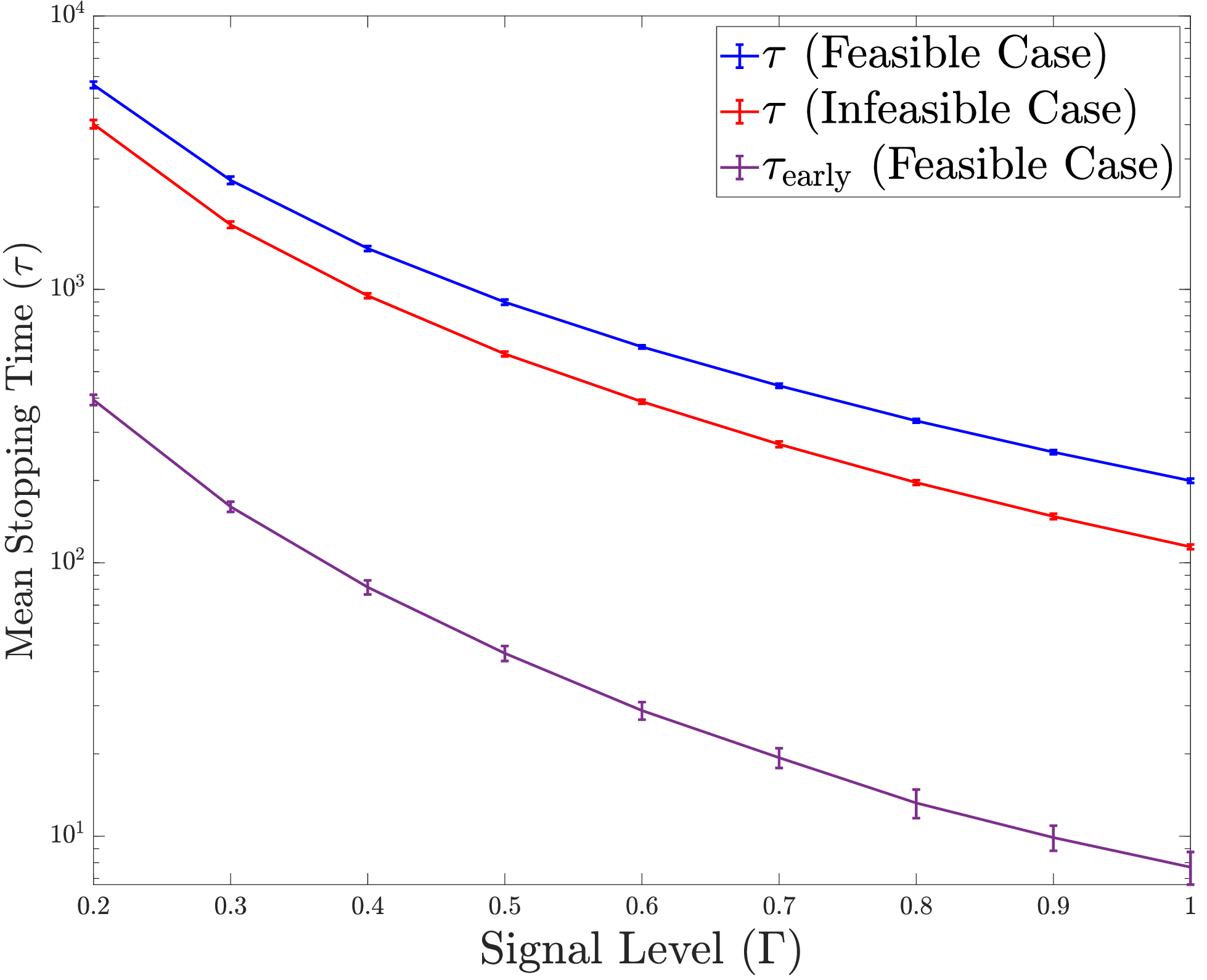} \vspace{-\baselineskip}
    \caption{\footnotesize Behaviour of the stopping time as $d$ is varied for fixed $\Gamma = 1/\sqrt{2}$ (\emph{left}) and  $\Gamma$ is varied for fixed $d = 4$ (\emph{right}) over the unit ball with $m = 2$. Averages and one-sigma error bars over 50 runs are reported. The test never returned an incorrect hypothesis. Notice the sharp advantage of $\tearly$ in feasible cases, in that it is about a factor of $10$ smaller than $\tau$. (\emph{best viewed zoomed-in})}\vspace{-\baselineskip}
    \label{fig:sims}
\end{figure}
\textbf{Observations} As a basic observation, we find that in \emph{all} runs, the test returns the correct hypothesis. Notice that this suggests that the testing boundary is overly conservative, and a finer analysis of the same is thus of interest. The main observation of Figure~\ref{fig:sims} is that for feasible instances $\tearly$ is typically $< \tau/10$, across all dimensions $d$ and signal level $\Gamma$ studied, indicating that this early stopping is very powerful. While the validity of stopping at time $\tearly$ is easy to see from the consistency of confidence sets, nothing in our analysis indicates the sample advantage of this procedure, and the resolving this is a natural open question.

\section{Discussion}

The feasibility testing problem is a natural first step prior to executing constrained bandit methods, and by initiating the study of the same, our work extends the applicability of this emerging field. We presented simple tests based on existing technology of online linear regression and LILs that are effective for such problems, and further pointed out key deficiencies in the extant work on the single-constraint finite-armed theory of this problem. Naturally, this is only a first step: the real power of the finite-armed theory, and in particular the tests proposed therein, is its strong adaptation to the explicit structure of the instance at hand. A parallel theory, both in the small and moderate $\delta$ regimes, in the linear setting is critical to develop efficient tests. Naturally, the computational question of how one can implement such tests efficiently is also critical. We hope that our work will spur study on these interesting and important issues.

\section*{Acknowledgements}

Aditya Gangrade was supported on AFRLGrant FA8650-22-C1039 and NSF grants CCF-2007350 and CCF-2008074. Aditya Gopalan acknowledges partial support from Sony Research India Pvt. Ltd. under the sponsored project `Black-box Assessment of Recommendation Systems'. Clayton Scott was supported in part by NSF Grant CCF-2008074. Venkatesh Saligrama was supported by the Army Research Office Grant W911NF2110246, AFRLGrant FA8650-22-C1039, the National Science Foundation grants CCF-2007350 and CCF-1955981.

\bibliography{SLB_feasibility}
\bibliographystyle{icml2024}

\newpage
\appendix
\onecolumn

\section{Tools from the Theory of Online Linear Regression and Linear Bandits}

As is standard in the setting of linear bandits, we shall exploit tools from the theory of online linear regression to enable learning and exploration. The main tool we use is Lemma~\ref{lemma:online_linear_regression}, stated previously in the main text, which asserts that the confidence sets $\confset_t$ are consistent with high probability, and control the deviations of $\tA x - Ax$ for $\tA \in \confset_t$ to the level $\rho_t(x;\delta)$ if $A \in \confset_t(\delta)$. The latter result is almost trivial: by the triangle and Cauchy-Schwarz inequalities, for any $\tA \in \confset_t(\delta), i \in [1:m],$ \[ |(\tA - A)x)^i| \le |( ((\tA - \hat{A}_t) x)^i| \le 2 \sup_{\tA \in \confset_t(\delta)} | (\tA^i - \hat{A}_t)^\top x| \le \sup_{\tA \in \confset_t(\delta)} |\tA^i - \hat{A}_t^i\|_{V_t} \|x\|_{V_t^{-1}} \le \omega_t(\delta)\|x\|_{V_t^{-1}} = \rho_t(x;\delta),\] where the final inequality is by definition of the confidence set $\confset_t(\delta)$. 

The principal way to use this bound is through the following generic control on the behaviour of $\det V_t$ and on $\sum_{s \le t} \rho_s(x_s;\delta).$ We again refer to \citet{abbasi2011improved}, although the result is older. See their paper for a historical discussion. \begin{lemma}\label{lemma:olr_cumulative_control}
    For any sequence of actions $\{x_t\} \subset \{\|x\| \le 1\},$ and any $t\ge 0$, it holds that \[ \log \det V_{t+1} \le \sum_{s = 1}^t \|x_s\|_{{V_s}^{-1}}^2 \le 2 \log \det V_{t+1} \le 2 d \log(1 + (t+1)/d).\] As a consequence, \[ \sum_{s \le t} \rho_s(x_s;\delta)^2 \le 2\omega_t(\delta)^2d \log(1 + (t+1)/d) \le 3d^2\log^2 (1 + (t+1)/d) + 6d \log(1 + (t+1)/d)(1 + \log(m/\delta)), \] and \[ \sum_{s\le t} \rho_s(x_s;\delta) \le \sqrt{t \sum \rho_s(x_s;\delta)^2} \le \sqrt{2dt \log(1 + (t+1)/d) \omega_t(\delta)^2}. \]
\end{lemma}

We will also find it useful to state the consistency of the confidence set in the following dual way \begin{lemma}\label{lemma:confset_tail_bound_style}
    For any sequence of actions $\{x_t\},$ and any $v > 0,$ it holds that \[\mathbb{P}\left(\exists t, i: \| \hat{A}_t^i - A^i\|_{V_t} \ge 1 + \sqrt{\frac{d}{4} \log\left( 1 + \frac{t}{d}\right) + \frac{1}{2}\log m + \frac{v}{2}} \right) \le \exp(-v) .\] \begin{proof}
        Since, by the first statement of Lemma~\ref{lemma:olr_cumulative_control}, $\log\det V_t = \log \det V_{(t-1) + 1} \le d \log(1 + t/d),$ it follows that \[ \omega_t(\delta) =  1 + \sqrt{\frac12 \log \frac{m}{\delta} + \frac14 \log \det V_t} \le 1 + \sqrt{\frac{d}{4} \log (1 + t/d) + \frac12 \log m + \frac12 \log(1/\delta)} =: \tilde{\omega}_t(\delta).\]  Now the claim follows by just noting that \[ \PP(\exists t, i: \|A^i - \hat{A}_t^i\|_{V_t} \ge \tilde\omega_t(\delta)) \le \PP(\exists t, i : \|A^i - \hat{A}_t^i\| \ge \omega_t(\delta)) \le \delta,\] and inverting the form of the upper bound obtained after expressing $\tilde\omega_t(\delta)$ as we have above.
    \end{proof}
\end{lemma}

\section{Analysis of \textsc{eogt}.}\label{appendix:eogt_analysis}

We will proceed to flesh out the analysis sketched in \S\ref{sec:analysis_sketch}, and show the relevant results. 

\subsection{Adpting the LIL to the Noise Process of \textsc{eogt}, and Control on the Rejection Timescale Bound. } \label{appendix:auxiliary_lemmata}

We begin arguing the following simple observation that extends the LIL to our situation. \begin{lemma}\label{lemma:lil_actual}
    For $i_t$ as chosen in \textsc{eogt} or \textsc{t-eogt}, it holds that $\{\eta_t^{i_t}\}$ forms a conditionally centred and $1$-subGaussian process with respect to the filtration generated by $\{(i_s, x_s, S_s)\}_{s \le t} \cup \{(x_t, i_t)\}.$ Therefore, for $Z_t := \sum_{s \le t} \zeta_s^{i_s},$ and any $\delta \in (0,1),$ it holds that $\PP(\exists t : |Z_t| > \lil(t,\delta)) \le \delta$.
    \begin{proof}
        We simply observe that $(x_t,i_t)$ are predictable given $\hist_{t-1} = (\{(x_s, S_s)\}_{s \le t-1})$. Thus, the sigma algebra generated by $\{(x_s,i_s, S_s)_{s\le t} \cup \{(x_t,i_t)\}$ is the same as that generated by $\hist_{t-1},$ and $\zeta_t$ is assumed to be conditionally centred and 1-subGaussian with respect to this filtration, and thus its predictable projection $\zeta_t^{i_t}$ inherits this property. The second claim is then immediate from Lemma~\ref{lemma:LIL}.  
    \end{proof}
\end{lemma}

We further add the proof of the upper bound on $T(\Gamma; \delta, N),$ which bounds the timescale of rejection for \textsc{eogt}.

\begin{proof}[Proof of Lemma~\ref{lemma:eogt_timescale}]

    We note that we shall make no efforts to optimise the constants in the following argument. Recall that \[ T(\Gamma;\delta,N) = \inf\left\{t \ge 2d: t |\Gamma| > 2\lil(t,\delta/2) + 4d\log(2t/d) \sqrt{t} + 2\sqrt{d t \log(2t/d) \log \frac{2m}{\delta t^{-N}}}\right\}.\]  
    
    Now, if $t \ge \max(50,2d),$ then \begin{align*} \frac{2 \lil(t,\delta/2)}{\sqrt{t}} &= 4 \sqrt{  \log(11\log t) + \log \frac{2}{\delta} } <  4\sqrt{\log t + \log \frac{2}{\delta}} \\ &\le 4 \sqrt{N \log t + \log(2m/\delta)} \\ &\le 4\sqrt{d  \log(2t/d) \log(2m/\delta t^{-N})},\end{align*} where we have used $N \ge 1$, that $\log(11 \log(u))< \log(u)$ for $u \ge 50,$ and that $d \log(2t/d) > 1$ when $2t/d > 4 > e$. Thus absorbing the $\lil$ term into the last term defining $T$, we conclude that  \begin{align*} T(\Gamma;\delta,N) &\le  \inf\left\{t \ge \max(50,2d): t|\Gamma| >  4d\sqrt{t}\log t +  6\sqrt{dt \log(2t/d)( \log(2m/\delta) + N\log t)}\right\} \\
    &\le  \inf\left\{t \ge \max(50,2d): t|\Gamma| >   \max\left(12 d\sqrt{t}\log t, 18\sqrt{dt \log(2t/d) \log(2m/\delta)}, 18\sqrt{dtN} \log t\right)\right\}\\
    &\le \inf\left\{t \ge \max(50,2d): \frac{t}{\log^2 t} >   \frac{12^2 \max(d^2,\nicefrac{9}{4} Nd)}{\Gamma^2} \textrm{ and } \frac{2t/d}{\log(2t/d)} > \frac{2 \cdot 18^2\log(2m/\Gamma)}{\Gamma^2} \right\}, \end{align*} where in the second step we used the facts that for $u,v,w \ge 0$, $\sqrt{u + v} \le \sqrt{u} + \sqrt{v}$ and $(u+v+w) \le 3\max(u,v,w)$. 

    Now, we observe the following elementary properties. \begin{enumerate}
        \item The map $u \mapsto u/\log(u)$ is increasing for $u \ge 3$. Thus, if $t > 2z \log 2z$ for some $z \ge 1.5$ (which implies $2z \log 2z \ge 3$), then \[ \frac{t}{\log t} > \frac{2z \log(2z)}{\log2z + \log\log(2z)} \ge z,\] where we have used that $2z > 1$ for $z \ge 1.5$. Since $\frac{2 \cdot 12^2 \log(2m/\delta)}{\Gamma^2} > 2 \cdot 12^2 \cdot \log(2) > 1.5,$ \[ 2t/d > \frac{4\cdot 18^2 \log(2m/\delta)}{\Gamma^2} \log \frac{4\cdot 18^2 \log(2m/\delta)}{\Gamma^2} \implies \frac{2t/d}{\log(2t/d)} > \frac{2 \cdot 18^2 \log(2m/\delta)}{\Gamma^2}. \]

        \item For $u > 1, v> 0,$ \[ \frac{u}{\log^2 u} \ge v \iff \left( \frac{\sqrt{u}}{2\log \sqrt{u}}\right)^2 \ge v \iff \frac{\sqrt{u}}{\log \sqrt{u}} \ge \sqrt{4v}. \] But, as detailed above, if $\sqrt{4v} > 3/2 \iff v > 9/16,$ then it holds for any $u$ such that \[ \sqrt{u} > 2\cdot \sqrt{4v} \log(2\cdot \sqrt{4v}) \iff u > 4v \log^2(16v). \] Setting $u =t, v = \frac{12^2 \max(d^2,\nicefrac{9}{4} Nd)}{\Gamma^2} > 9/16,$ we conclude that \[t > \frac{4 \cdot 12^2 \max(d^2,\nicefrac{9}{4} Nd)}{\Gamma^2} \log^2 \frac{16 \cdot 12^2 \max(d^2,\nicefrac{9}{4} Nd)}{\Gamma^2} \implies \frac{t}{\log^2 t} > \frac{12^2 \max(d^2,\nicefrac{9}{4} Nd)}{\Gamma^2}.  \]
     \end{enumerate}

     Incorporating the above analysis into the bound on $T(\Gamma;\delta,N),$ we conclude that \[ T(\Gamma;\delta,N) \le \max\left(50, 2d, \frac{576 \max(d^2, \nicefrac{9}{4} Nd)}{\Gamma^2} \log^2 \frac{2304\max(d^2,\nicefrac{9}{4} Nd)}{\Gamma^2} , \frac{648 d\log(2m/\delta)}{\Gamma^2} \log \frac{1296 \log(2m/\delta)}{\Gamma^2}\right). \qedhere \]
    
\end{proof} 

\subsection{Signal growth under consistency of confidence sets, and reliability}

The growth of $\stat_t$ was detailed in the main text in \S\ref{sec:analysis_sketch}, the only informal aspect of this section being the treatment of $Z_t$, which can be accounted for immediately using Lemma~\ref{lemma:lil_actual}. Thus, we have already shown Lemma~\ref{lemma:boundary_eogt}. As briefly mentioned in the main text, this immediately yields reliability. \begin{proposition}\label{lemma:eogt_reliability}
    \textsc{eogt} is reliable. 
    \begin{proof}
        Suppose that $\hfeas$ is true, and the event of Lemma~\ref{lemma:boundary_eogt} holds. Then since $\tau = \inf\{t : |\stat_t| > \boundary_t(\delta)\},$ and since $\stat_t \ge -\boundary_t(\delta),$ it follows that upon stopping, $\stat_\tau > \boundary_t(\delta).$ Since $\mathscr{D}(\hist_\tau) = \hfeas$ if $\stat_\tau > 0,$ it follows that this decision is correct. Hence, the only way for the decision to be incorrect is if $\exists t : \stat_t < t\Gamma - \boundary_t(\delta),$ which can occur with probability at most $\delta$. The same argument can be repeated mutatis mutandis for $\hinfeas$. 
    \end{proof}
\end{proposition}

\subsection{Control on the Stopping Time of \textsc{eogt} in Mean and Tails}\label{appendix:eogt_tail}

We shall prove the stronger result, Theorem~\ref{theorem:eogt_tail}. Note that expectation result follows from this directly.

\begin{proof}[Proof of Theorem~\ref{theorem:eogt_main_result} assuming Theorem~\ref{theorem:eogt_tail}]
    The reliability has already been shown in Proposition~\ref{lemma:eogt_reliability}. To control the expectation, let us define, for naturals $k\ge 2$, $T_k = T(\Gamma/2; \delta/2^{k^3 - 1};N) + \lceil 2^{(k^3-1)/N}\rceil (2 + 1/|\Gamma|),$ and define $T_1 = T(\Gamma;\delta,N)$.  Then by Theorem~\ref{theorem:eogt_tail}, $\PP(\tau > T_k) \le 2^{-(k^3-1)}\delta.$ As a consequence, \begin{align*} \mathbb{E}[\tau] &= \sum_{t \ge 0} \mathbb{P}(\tau > t) \\ &\le \sum_{t \le T_1} P(\tau > t) + \sum_{k = 2}^\infty \sum_{t \in [T_{k-1} +1:T_k]} \PP(\tau > t) \\ &\le T_1 + 1 +  \sum_{k = 2}^\infty \delta 2^{1-(k-1)^3}(T_k - T_{k-1}) \le T_1 + 1 + \delta \sum_{k = 2}^\infty  T_k 2^{1 - (k-1)^3}.\end{align*}

    To control the above, we shall show that $T(\Gamma;\delta^{k^3},N)$ is bounded from above by $k^6 T(\Gamma;\delta,N)$ for $k \ge 2$. To this end, recall that \[T(\Gamma;\eta,N) = \inf\left\{t \ge 2d: t|\Gamma| > 4\sqrt{t\log \log t + t\log \frac{22 m}{\delta}} + 4d\sqrt{t}\log(2t/d) +  \sqrt{2 d t\log(2t/d) \log(2m/\delta t^{-N})}  \right\}. \]

    Now, first observe that if $t \ge 16,$ and $k \ge 2$, then \( \log\log(k^6t) \le k^6 \log\log t\). Indeed, if $k \ge t,$ then\footnote{$\log\log k^7 = \log 7 + \log\log k \le \log 7+ \log k - 1 \le \log 7 -2 + k,$ and $e^2 > 7.3$.  } $\log\log(k^6 t) \le \log\log(k^7) \le k < k^6$.  If instead $k \le t,$ then $\log(7) < 2 < (k^6-1) \implies \log \log t^7 = \log 7 + \log \log t < k^6 -1 + \log\log t < k^6 \log\log t,$ which exploits that $\log\log t > 1$ for $t\ge 16> e^e$. 
    
    Next, if $t \ge 2d,$ and $k \ge 2,$ then $ \log (2k^6 t/d) \le k^3 \log(2t/d).$ Again, if $2t/d < k,$ then\footnote{$k^3 - 7k + 7$ is growing for $k \ge \sqrt{7/3} \approx 1.52$, and $2^3 - 14 + 7 = 1 > 0$. Of course, $\log(4) > 1$.  } $\log(k^7) < 7(k-1) < k^3 \log(4) < k^3 \log(2t/d)$, and if $2t/d \ge k,$ then $7\log(2t/d) < k^3\log(2t/d)$ since $k^3 \ge 8$. Similarly, if $k \ge 2, t \ge 16$ then $\log(k^6t) \le k^3 \log t$.

    It follows from the above that if $t \ge \max(2d,16),$ and $t \ge T(\Gamma;\delta, N),$ then $k^6 t \ge T(\Gamma; \delta/2^{k^2-1},N).$ Indeed, since $t  > T(\Gamma; \delta,N),$ we have \[t |\Gamma| > 4\sqrt{t \log\log t+ \log(22m/\delta)} + 4d\sqrt{t} \log(2t/d) + \sqrt{(2d\log(2t/d) (\log(2m/\delta) + N \log t)}. \] Multiplying through by $k^6,$ and using $t \ge \max(2d, 16),$ we observe that  \begin{align*} k^6 t |\Gamma| &> 4\sqrt{k^6 t ( k^6 \log \log t + k^6 \log \frac{22 m}{\delta}) } + 4d \sqrt{k^6 t} \cdot k^3\log(2t/d) \\ &\qquad\qquad\qquad +  \sqrt{2 d(k^6 t) \cdot k^3\log(2t/d) ( k^3\log(2m/\delta) + 2Ndk^3 \log t)} \\
    &\ge 4\sqrt{(k^6 t) \log \log (k^6 t) + \log \frac{22 m}{\delta^{k^6}}} + 4d\sqrt{k^6 t} \log(2k^6t/d) \\ &\qquad\qquad\qquad+  \sqrt{2 d(k^6 t) \log(2k^6 t/d) (\log(2m/\delta^{k^3}) + 2Nd \log (k^6 t )  ) },
    \end{align*}
    where we have used that $m \ge 1.$ Since $\delta \le 1/2,$ \[ \delta^{k^6} \le \delta^{k^3} = \delta \cdot \delta^{k^3 - 1} \le \delta \cdot 2^{-(k^3 -1)}.\] Thus, we conclude that for $k \ge 2,$ \[ T_k - \lceil 2^{(k^3-1)/N}\rceil (2 + 1/|\Gamma|) = T(\Gamma/2; \delta/2^{k^3-1},N) \le \max(2d, 16, k^6 T(\Gamma/2; \delta;N)).\]
    
    Plugging this into the bound on $\mathbb{E}[\tau],$ we conclude using numerical estimates of the quickly converging series $\sum_{k \ge2 } 2^{1-(k-1)^3} \le 1.01$ and $\sum_{k \ge 2} k^6 2^{1-(k-1)^3} \le 70$ that 
    \begin{align*}
        \mathbb{E}[\tau] &\le T_1 + 1 + \delta \sum_{k \ge 2}  2^{1 - (k-1)^3} T_k \\
                          &\le T_1 + 1 + \delta \sum_{k \ge 2} 2^{1 - (k-1)^3} (2d + 16) + \delta T(\Gamma/2;\delta,N) \sum_{k \ge 2} k^6 2^{1 - (k-1)^3}  \\&\qquad + \delta(2 + 1/|\Gamma|) \left( \sum_{k \ge 2} 2^{- ( (k-1)^3-1 -(1-1/N) k^3) } + 2^{1 -(k-1)^3}\right) \\
                          &\le T(\Gamma;\delta,N) +  1 +  70 \delta T(\Gamma/2;\delta,N) +   (20 + 3d)\delta + O(1)\delta (3 + 1/|\Gamma|),   
    \end{align*}
    where the $O(1)$ term is $\le 1.01 + \sum_{k\ge 2} 2^{ 1- (k-1)^3 + (k^3(1-1/N))},$ which is summable since $N > 1$.
\end{proof}

Let us now proceed with the 
\begin{proof}[Proof of Theorem~\ref{theorem:eogt_tail}]

    First notice by Lemma~\ref{lemma:olr_cumulative_control}, if $t \ge 2d,$ then \[ \sum_{s \le t} \rho_s(x_s;\delta_s/2) \le \sum_{s \le t} \rho_s(x_s;\delta_t/2) \le \omega_t(\delta_t/2)\sqrt{2dt \log(1 + (t+1)/d)} \le \omega_t(\delta_t/2)\sqrt{2 d t \log(2t/d)}.\] Consequently, we have that for $t \ge 2d,$ \[ \boundary_t(\delta) \le \omega_t(\delta_t/2)\sqrt{2dt \log(1 + (t+1)/d)} \le \omega_t(\delta_t/2) \sqrt{2 d t\log(2t/d)}. \] If $\hfeas$ is true, then we know by Lemma~\ref{lemma:boundary_eogt} that with probability at least $1-\delta,$ \[ \forall t, \stat_t \ge t\Gamma - \boundary_t(\delta), \] and so we conclude that under this event, \[ \tau = \inf\{ t:  t \Gamma > 2\boundary_t(\delta) \}.\] But due to the deterministic upper bound on $\boundary_t(\delta)$ under the same event,  \[ \tau \le \inf\{t : t\Gamma > 2\omega_t(\delta_t/2)\sqrt{2 d t \log(2t/d)} + 2\lil(t;\delta/2)\}.\] But, for $t \ge 2d, N > 1,$ \( \omega_t(\delta_t/2) \le 1 + \sqrt{\frac12\log(2m/\delta t^{-N}) + \frac{d}{4} \log(2t/d)},\) and so, \begin{align*} 2\omega_t(\delta_t/2) \sqrt{2d t \log 2t/d} &\le 2\sqrt{2dt \log(2t/d)} + 2\sqrt{\frac{d^2}{2} t\log^2(2t/d)} + 2\sqrt{(d t\log(2t/d)) (\log(2m/\delta t^{-N}))} \\
    &\le (\sqrt{8/\log(4) d} + \sqrt{2}d)  \sqrt{t} \log(2t/d)  + 2\sqrt{(d \log(2t/d)) (\log(2m/\delta t^{-N}))} \\
    &< 4d \sqrt{t} \log(2t/d) + 2\sqrt{(d \log(2t/d)) (\log(2m/\delta t^{-N}))},\end{align*}  where the first line uses $\sqrt{u + v} \le \sqrt{u} + \sqrt{v},$ and the final line uses the fact that $t \ge 2d \implies \log(2t/d) \ge \log(4),$ and that for $u \ge 1, \sqrt{8u/\log 4} + \sqrt{2}u < 4u$. But this implies that \[ \tau \le \inf\left\{t : t|\Gamma| > 2 \lil(t,\delta/2) + 4d \sqrt{t} \log(2t/d) + 2\sqrt{d \log(2t/d)\log \frac{2m}{\delta t^{-N}}}\right\} = T(\Gamma;\delta,N). \] In fact, this is precisely why $T(\Gamma;\delta,N)$ was so defined. Thus, in the feasible case, with probability at least $1-\delta, \PP(\tau > T(\Gamma; \delta,N)) \le \delta$. The argument is identical in the infeasible case, barring sign flips.

    Control on the tail can be obtained by essentially bootstrapping the above result along with our choice of $\decset_t = \confset_t(\delta_t),$ the key idea being that since $\delta_t \to 0,$ for large enough $t, A \in \decset_t$ must actually occur with near-certainty. Formally, let us define $T_\eta = \inf\{t : \delta_t < \eta\} =\lceil (\delta/\eta)^{1/N}\rceil.$ Then notice that for every $t \ge T_\eta,$ it holds that $\decset_t \subset \confset_t(\eta),$ and so $\PP(\forall t \ge T_\eta, A \in \decset_t) \ge 1-\eta$. Therefore, repeating the proof of Lemma~\ref{lemma:boundary_eogt}, we conclude that in the feasible case, for all $t \ge T_\eta,$ \[\stat_t \ge -T_\eta + (t - T_\eta)\Gamma - \sum_{T_\eta \le s \le t} \rho_s(x_s;\delta_s) - \lil(t,\eta/2),\] where we have used the fact that $\|x\|\le 1, \|A^i\| \le 1$ to conclude that $|(Ax)^{i_t}| \le 1$ in order to handle the times $t \in [1:T_\eta-1].$ In particular, if $t > 2T_\eta + T_\eta/\Gamma,$ then \( \stat_t \ge t \frac{\Gamma}{2} - \boundary_t(\eta).\) 
    
    But we know that we must stop before time $t$ if \( \stat_t \ge \boundary_t(\delta), \) and since $\boundary_t(\delta) \le \boundary_t(\eta)$ uniformly, we conclude that under the event that $A \in \decset_t $ for all $t \ge T_\eta,$ then it must hold that \[ \tau \le \max\left( (2+ 1/\Gamma)T_\eta,  T(\Gamma/2, \eta, N) \right).\] Since this occurs with probability at least $1-\eta,$ the conclusion follows for the feasible case. Again, the argument is identical for the infeasible case, barring sign flips.   
\end{proof}

\section{Analysis of \textsc{t-eogt}}\label{appendix:tempered_eogt_proofs}

The main result follows simply from the key control offered in Lemma~\ref{lemma:tempered_boundary_derivation}, and showing the latter will form the bulk of this section. We proceed by first showing the stopping time bounds.

\begin{proof}[Proof of Theorem~\ref{theorem:tempered_test}]
    Let us consider the feasible case; the infeasible case follows similarly. For reliability, observe that via Lemma~\ref{lemma:tempered_boundary_derivation}, it holds with probability at least $1-\delta$ that for all $t$, \[ \tstat_t \ge t\Gamma - \mathscr{Q}^{\mathsf{F}}_t(\delta) > - \mathscr{Q}^{\mathsf{F}}_t(\delta). \] Since the stopping time is \[ \ttau = \inf\{t : \tstat_t < - \mathscr{Q}^{\mathsf{F}}_t(\delta) \text{ or } \tstat_t > \qinfeas_t(\delta)\}, \] it follows that if the preceding event occurs, then if the test stops, it must be correct. But, since $\qinfeas+ \qfeas$ grows sublinearly in $t$, under the same event the test must eventually stop. Therefore, the probability that we stop and make an error is bounded by $\delta,$ making the test reliable.

    It remains to control the behaviour of $\ttau$. To this end, again observe that for any $\eta \in (0,1),$ with probability at least $1-\eta,$ it holds for all time that \[ \tstat_t \ge t \Gamma - \qfeas_t(\eta).\] Thus, we conclude that with probability at least $1-\eta,$ \[ \tau \le \inf\{t : t\Gamma \ge \qfeas_t(\eta) + \qinfeas_t(\delta)\} \le T_\eta := \inf\{ t : t\Gamma \ge \qfeas_t(\eta) + \qinfeas_t(\eta)\}. \] 

    But notice that \[ \qfeas_t(\eta) + \qinfeas_t(\eta) \le 50 t^{\nicefrac12} \log^2(t) \left( d^{3/2} + d^{1/2} \log(8m/\eta)\right) + 2\lil(t,\eta/2).\] Following the approach in the proof of Lemma~\ref{lemma:eogt_timescale} as presented in \S\ref{appendix:auxiliary_lemmata},\footnote{the only new information needed being that $4\log\log z \le \log z$ for all $z \ge 2$} we immediately get that there exists a constant $C$ such that with probability at least $1-\eta,$ \begin{align*}
        \tau \le \frac{C \log(C\log(\Gamma^{-2})/\delta)}{\Gamma^2} + \frac{C d^3}{\Gamma^2} \log^4 \frac{Cd^3}{\Gamma^2} + \frac{C d \log(8m/\eta)}{\Gamma^2} \log^4 \frac{d\log(8m/\eta)}{\Gamma^2}.
    \end{align*} 

    The expectation bound is immediate upon integrating the tail. 
\end{proof}

It remains then to show Lemma~\ref{lemma:tempered_boundary_derivation}, which is the subject of the next section.

\subsection{Proof of Anytime Behaviour of \texorpdfstring{$\widetilde{\stat}_t$}{the Test Statistic}}\label{appendix:tempered_eogt_tail}

\newcommand{\reg}{\mathscr{R}}

We begin with setting up some notation, and then proceed by explicitly describing key observations underlying the argument, encapsulated as lemmata. The key aspects of this argument follow the analysis of \citet{simchi2023regret}.

\subsubsection{Notation}\label{sec:notation}
Let $(x^*, i^*)$ denote any solution to the program $\max_x \min_i (Ax)^i,$ which we shall fix for the remainder of this section. Of course, $(Ax^*)^{i^*} = \Gamma$. Recall that $\imin(x) = \argmin_i (Ax)^i.$ We further define \[ i_t(x) = \argmin_i (\hA_t x)^i, \textit{ and } i^*_t = i_t(x^*).\] We denote the estimation error in $\hA_t$ as \[ B_t = \hA_t - A. \] Next, we define the random quantity \[ \Delta_t = (\Gamma - (Ax_t)^{i_t})\mathrm{sign}(\Gamma) = \begin{cases}\Gamma - (Ax_t)^{i_t} & \textrm{if }\Gamma > 0, \textrm{ i.e., under feasibility} \\ (Ax_t)^{i_t} - \Gamma & \textrm{if } \Gamma < 0, \textrm{ i.e., under infeasibility}. \end{cases}, \] and the cumulative pseduoregret-like object \[ \mathscr{R}_t = \sum_{s \le t } \Delta_s.\] The point here is that we may decompose  \begin{alignat*}{3}
    \tstat_t &= \sum_{s \le t} (Ax_s)^{i_s} + Z_t =  t\Gamma -  \reg_t + Z_t, &&\quad \textit{(feasible case)}\\
    \tstat_t &= \sum_{s \le t} (Ax_s)^{i_s} + Z_t = t\Gamma + \reg_t + Z_t, &&\quad \textit{(infeasible case)}
\end{alignat*} and thus in either case, if we show that $\reg_t$ is not too large, then $\tstat_t$ has favourable behaviour. Observe that if we were working in a single objective setting, $m=1$, then in the feasible case $\reg_t$ would be the pseudoregret of a linear bandit instance.

Since these quantities will appear often in the argument, we further define \[ N_t = \|x_t\|_{V_t^{-1}}^2 \textit{ and } N_t^* = \|x^*\|_{V_t^{-1}}^2,\] and for $v \ge 0,$ \[ \mathscr{W}_t(v) := 1 + \sqrt{\frac d4 \log(1 + t/d) + \frac12 \log m + \frac v2}  \] Finally, notice that with the above notation, Lemma~\ref{lemma:confset_tail_bound_style} can be expressed as \[ \forall v > 0, \PP\left(\exists t, i : \|B_t^i\|_{V_t} \ge \mathscr{W}_t(v) \right) \le e^{-v}.\] Further, \[ \mathrm{Rad}_t(x_t) = (t/d)^{\nicefrac12}N_t + \sqrt{dN_t}, \textit{ and } \mathrm{Rad}_t(x^*) = (t/d)^{\nicefrac12} N_t^* + \sqrt{dN_t^*}.\]

\subsubsection{Structural Observations} The following two results constitute basic structural observations due to \citet{simchi2023regret} that enable the subsequent analysis. The first argues that in each round, some quantity of the form $(B_t x)^{i}$ for some $(x,i)$ is large in absolute value.

\newcommand{\itstar}{i_t^*}

\begin{lemma}\label{lemma:delta_separation}
    For the sequence of actions $\{x_t\}$ selected by \textsc{t-eogt}, the following hold. \begin{itemize}
        \item In the feasible case, at each time, either the first or the second of the following hold: \begin{align*}
            (B_t x_t)^{i_t} &\ge \Delta_t/2 - (t/d)^{\nicefrac12} N_t - \sqrt{d N_t} \\ 
           \textit{ or} \quad -(B_t x^*)^{i_t^*} &\ge \Delta_t/2 + (t/d)^{\nicefrac12} N_t^* + \sqrt{d N_t^*} 
        \end{align*}  
        \item In the infesible case, at each time $t$, either the first or the second of the following hold: \begin{align*}
            - (B_t x_t)^{i_t} &\ge \Delta_t/2 \\ \textit{or } \quad (B_t x_t)^{\imin(x_t)} &\ge \Delta_t/2.
        \end{align*}
    \end{itemize}
    \begin{proof}
        In the feasible case, due to the optimistic selection, it must hold that \[ (\hA_t x_t)^{i_t} + \mathrm{Rad}_t(x_t) \ge (\hA_t x^*)^{i_t^*} + \mathrm{Rad}_t^*. \] Now, we may write $\hat{A}_t = A + B_t,$ and so get \[ (B_t x_t)^{i_t} + \mathrm{Rad}(x_t) \ge \left( (Ax^*)^{\itstar}  - (Ax_t)^{i_t} \right) + \mathrm{Rad}_t(x^*).\] But note that $(A x^*)^{\itstar} \ge \min_i (Ax^*)^i = \Gamma,$ and so $(Ax^*)^{\itstar} - (Ax_t)^{i_t} \ge \Delta_t$ in the feasible case. Thus, we have \[(B_t x_t)^{i_t} + \mathrm{Rad}(x_t) \ge \Delta_t + (B_t x^*)^{\itstar} + \mathrm{Rad}_t(x^*). \] But, since if $A \ge B + C,$ then either $A \ge B/2$ or $-C \ge B/2$, it follows that at least one of the following must hold: \[ (B_t x_t)^{i_t} \ge \Delta_t/2 - \mathrm{Rad}_t(x_t) \textit{ or } -(B_t x^*)^{\itstar} \ge \Delta_t/2 + \mathrm{Rad}_t(x^*). \] The conclusion follows upon incorporating the form of $\mathrm{Rad}_t(x_t)$ and $\mathrm{Rad}_t(x^*)$ indicated before the statement of the lemma.

        In the infesible case, we note that it must hold that \[ (\hA_t x_t)^{i_t} \le (\hA_t x_t)^{\imin(x_t)} \iff (B_t x_t)^{i_t} - (B_t x_t)^{\imin(x_t)} \ge (Ax_t)^{i_t)} - (Ax_t)^{\imin(x_t)}.\] But, $(Ax_t)^{\imin(x_t)} = \min_i (Ax_t)^i \le \max_x \min_i (Ax)^i = \Gamma,$ and so noting that $\Delta_t = (Ax_t)^{i_t} - \Gamma$ in the infeasible case, we have \[ (B_t x_t)^{i_t} - (B_t x_t)^{\imin(x_t)} \ge \Delta_t, \] which again yields the conclusion. 
    \end{proof}
\end{lemma}

The next observation essentially yields a condition for low $\reg_t$ in terms of $(\Delta_t, N_t)$, and forms a refinement of the key observation of \citet{simchi2023regret} that allows us to extend their results to yield anytime bounds.

\begin{lemma}\label{lemma:low_reg_condition}
    For any nondecreasing sequence of positive reals $u_t,$ it holds that \[ \{ \exists t : \reg_t > u_t (1 + \log( t+1))\} \subset \{ \exists t : \Delta_t \ge u_t/3t, N_t < d/t \} \cup \{ \exists t: \Delta_t/N_t \ge u_t/3d, N_t \ge d/t\}. \]
    \begin{proof}
        Suppose that for all $t, N_t < d/t \implies \Delta_t < u_t/3t$ and $N_t \ge d/t \implies \Delta_t/N_t < u_t/3d$. Then
       \begin{align*} \reg_t &= \sum_{s \le t} \Delta_s = \sum_{s \le t} \Delta_s \indi\{N_s < d/t\} + \sum_{s \le t} \frac{\Delta_s}{N_s} \cdot N_s \indi\{N_s \ge d/t\} \\
                             & < \sum_{s \le t} u_s/3s + \sum_{s \le t} \frac{u_s}{3d} N_s \\
                             &\le \frac{u_t}{3} \sum_{s \le t} 1/s + \frac{u_t}{3d}\sum_{s\le t} N_s\\
                             &\le \frac{u_t( \log(t) + 1)}{3} + \frac{u_t}{3d} \cdot 2d \log(1 + t/d)\\
                             &\le u_t (1 + \log(t+1)),\end{align*}
                              where the second inequality is because $u_s \le u_t$ for all $s \le t$, and the third uses the bound on $\sum_{s \le t}N_s = \sum_{s \le t} \|x_s\|_{V_s^{-1}}^2$ from Lemma~\ref{lemma:olr_cumulative_control}, and the standard bound on harmonic numbers $\sum_{s \le t} 1/s \le \log(t) + 1$. 
    \end{proof}
\end{lemma}

This sets up the basic approach: the two events in Lemma~\ref{lemma:delta_separation} along with the two events in Lemma~\ref{lemma:low_reg_condition} set up four potential ways that high $\reg_t$ can arise in either the feasible or the infeasible case. We will separately bound the probabilities of these events by repeated reduction to the key result of Lemma~\ref{lemma:confset_tail_bound_style}.

\subsubsection{Controlling the Chance of Poor Events}

We now proceed to execute the strategy we described at the end of the previous section. We will separate the arguments for the feasilbe and the infeasible cases.

\paragraph{Feasible Case} We shall further separate the analysis into two cases, depending on if we control the event with $|(B_t x_t)^{i_t}|$ being large, or $|(B_t x^*)^{i_t^*}|$ being large.

\begin{lemma}\label{lemma:poor_events_A}
    For any $v \ge 0,$ define  \[ U_t^{\mathsf{F, A}}(v) :=  6\sqrt{dt}  + 6d\sqrt{t} + 6\sqrt{dt} \mathscr{W}_t(v).  \] Then both of the following inequalities hold true: \begin{align*} \PP(\exists t : \Delta_t \ge U_t^{\mathsf{F, A}}(v)/3t, N_t < d/t, (B_tx_t)^{i_t} \ge \Delta_t/2 - (t/d)^{\nicefrac12}N_t - \sqrt{dN_t})  &\le e^{-v}, \\
    \PP(\exists t : \Delta_t/N_t \ge U_t^{\mathsf{F, A}}(v)/3d, N_t \ge d/t, (B_t x_t)^{i_t} \ge \Delta_t/2 - (t/d)^{\nicefrac12} N_t - \sqrt{dN_t}) &\le e^{-v}. \end{align*}
\end{lemma}
\begin{proof}
    We argue the two inequalities using slightly different, but ultimatly similar approaches. The key observation we will need is that by the Cauchy-Schwarz inequality, and since $N_t = \|x_t\|_{V_t^{-1}}^2, |(B_t x_t)^{i_t}| = |(B_t^{i_t} V_t^{1/2} V_t^{-1/2} x_t)| \le \|B_t^{i_t}\|_{V_t} \sqrt{N_t}$. Throughout, we will let $u_t$ denote an arbitrary nondecreasing sequence, and derive the form of $U_t^{\mathsf{F, A}}$ at the end.
    
    \noindent \emph{Case (i)}.  Suppose $\Delta_t \ge u_t/3t$ and $N_t < d/t.$ Then \begin{align*} (B_t x_t)^{i_t} &\ge \frac{\Delta_t}{2} - \sqrt{\frac t d}N_t - \sqrt{d N_t} \\   
                                                                                                         &\ge \frac{u_t}{6t} - \sqrt{\frac  d t} - \frac{d}{\sqrt{t}} \\  
                                                                                \implies \sqrt{N_t} \| B_t^{i_t} \|_{V_t}   &\ge  \frac{ u_t - 6\sqrt{dt} - d\sqrt{t} } {6t}  \\  
                                                                                \implies \|B_t^{i_t}\|_{V_t} &\ge \frac{u_t - 6\sqrt{dt} - 6d\sqrt{t}}{6\sqrt{dt}}.                 \end{align*}       
    \noindent \emph{Case (ii)}. If instead, $\Delta_t/N_t \ge u_t/3d$ and $N_t\ge d/t,$ then\begin{align*} (B_t x_t)^{i_t} &\ge \frac{\Delta_t}{2} - \sqrt{\frac t d}N_t - \sqrt{d N_t} \\
                                                                                \iff          (B_t x_t)^{i_t}/N_t &\ge \frac{\Delta_t}{2N_t} - \sqrt{\frac t d} - \sqrt{d/N_t}\\
                                                                                \implies          \|B_t^{i_t}\|_{V_t}/\sqrt{N_t} &\ge \frac{\Delta_t}{2N_t} - \sqrt{\frac t d} - \sqrt{d/N_t} \\
                                                                                \implies {\|B_t^{i_t}\|_{V_t}} &\ge \frac{u_t}{6d} \cdot \sqrt{d/t} - 1 - \sqrt{t} \\
                                                                                &= \frac{u_t - 6\sqrt{dt} - 6d\sqrt{t}}{6\sqrt{dt}} .
                                                                                \end{align*} 
    Now observe that due to the form of $U_t^{\mathsf{F, A}},$ it holds that \begin{align*}
    \frac{ U_t^{\mathsf{F, A}}(v) - 6\sqrt{dt} - 6d \sqrt{t}}{6\sqrt{dt}} = \mathscr{W}_t(v),
\end{align*} and so we have \[ \PP\left( \exists t : \|B_t^{i_t}\|_{V_t} \ge \frac{ U_t^{\mathsf{F, A}}(v) - 6\sqrt{dt} - 6d \sqrt{t}}{6\sqrt{dt}}\right) \le \PP\left( \exists t, i: \|B_t^i\|_{V_t} \ge \mathscr{W}_t(v) \right),\] and the claim follows by Lemma~\ref{lemma:confset_tail_bound_style}.
\end{proof} 

\begin{lemma}\label{lemma:poor_events_B}
    For any $v \ge 0,$ define \[ U_t^{\mathsf{F, B}}(v) := \frac{3\sqrt{dt}}{2} (\mathscr{W}_t(v) - \sqrt{d})_+^2  ,\] where $(z)_+^2 = (\max(z,0))^2.$  Then it holds that \begin{align*}
        \PP(\exists t : \Delta_t \ge U_t^{\mathsf{F, B}}(v)/3t, N_t < d/t, -(B_t x^*)^{i_t^*} \ge \Delta_t/2 + \sqrt{t/d} N_t^* + \sqrt{d N_t^*}) &\le e^{-v}\\
        \PP(\exists t : \Delta_t/N_t \ge U_t^{\mathsf{F, B}}(v)/3d, N_t \ge d/t, -(B_t x^*)^{i_t^*} \ge \Delta_t/2 + \sqrt{t/d} N_t^* + \sqrt{d N_t^*}) &\le e^{-v}
    \end{align*}
\end{lemma}
\begin{proof}
    As in the proof of Lemma~\ref{lemma:poor_events_A}, let $u_t \ge 0$ be any sequence. Then observe that $\Delta_t/N_t \ge u_t/3d, N_t \ge d/t \implies \Delta_t \ge u_t/3t$. Further, by the AM-GM inequality, \[ \frac{u_t}{6t} + \sqrt{\frac t d } N_t^* \ge 2 \sqrt{ \frac{u_t}{6\sqrt{dt}} N_t^* }. \] 

    But, if $\Delta_t \ge u_t/3t,$ then  
    \begin{align*} -(B_t x^*)^{i_t^*} &\ge \frac{u_t}{6t} + \frac t d N_t^* + \sqrt{d N_t^*} \ge 2 \sqrt{ \frac{u_t}{6\sqrt{dt}} N_t^* } + \sqrt{dN_t^*} \\
        \implies \|B_t^{i_t^*}\|_{V_t} &\ge \sqrt{ \frac{2u_t}{3\sqrt{dt}} } + \sqrt{d}. \end{align*}
    
    Now, $U_t^{\mathsf{F, B}}$ is chosen so that \[ \sqrt{ \frac{2U_t^{\mathsf{F, B}}(v) }{3\sqrt{dt}}} + \sqrt{d} = \mathscr{W}_t(v), \] therefore, both of the probabilities in the claim are bounded from above by $\PP(\exists t, i: \|B_t^{i}\|_{V_t} \ge \mathscr{W}_t(v)),$ and we may conclude using Lemma~\ref{lemma:confset_tail_bound_style}. 
\end{proof}

\paragraph{Infeasible Case} Turning now to the infeasible case, we have the somewhat simpler bound below. \begin{lemma}\label{lemma:poor_events_infeasible}
    For $v\ge 0,$ let \[ U_t^{\mathsf{I}}(v) := 6\sqrt{dt} \mathscr{W}_t(v).\] It holds that \begin{align*} \PP(\exists t, i : \Delta_t \ge U_t^{\mathsf{I}}(v)/3t, N_t < d/t, |(B_t x_t)^{i}| \ge \Delta_t/2) &\le e^{-v}\\\PP(\exists t, i : \Delta_t/N_t \ge U_t^{\mathsf{I}}(v)/3d, N_t \ge d/t, |(B_t x_t)^{i}| \ge \Delta_t/2) &\le e^{-v} \end{align*}
\end{lemma}
\begin{proof}
    The argument is similar to that underlying Lemma~\ref{lemma:poor_events_A}. Let $u_t$ be any positive real. Then
    
    \noindent\emph{Case (i)} If $\Delta_t \ge u_t/3t, N_t < d/t,$ then for any $i$,\begin{align*}
        |( B_t x_t)^i| &\ge \Delta_t/2 \\ \implies \|B_t^i\|_{V_t} \sqrt{N_t} &\ge \frac{u_t}{6t}\\
        \implies \sqrt{d/t} \|B_t^i\|_{V_t} &\ge \frac{u_t}{6t} \iff \|B_t^i\|_{V_t} \ge \frac{u_t}{6\sqrt{dt}}.
    \end{align*}
    \noindent \emph{Case (ii)} If instead $\Delta_t \ge u_t/3d, N_t \ge d/t,$ then note that \[ \Delta_t /\sqrt{N_t} = \Delta_t/N_t \cdot \sqrt{N_t} \ge \frac{u_t}{3d} \cdot \sqrt{d/t} = \frac{u_t}{3\sqrt{dt}},\] and thus \[ |(B_t x_t)^i| \ge \frac{\Delta_t}{2} \implies \|B_t^i\|_{V_t} \ge \frac{u_t}{6\sqrt{dt}}. \]

    Since $U_t^{\mathsf{I}}(v) = 6\sqrt{dt} \mathscr{W}_t(v),$ it again follows that either of the probaiblities in the claim are bounded by $\PP(\exists t, i : \|B_t^i\|_{V_t} \ge \mathscr{W}_t(v))$, and we are done upon applying Lemma~\ref{lemma:confset_tail_bound_style}. 
\end{proof}
    
\subsection{Proof of Tail Bounds}

We are now ready to prove the claim. We begin by summarising the previous section through the lemma below. Note that setting $m = 1,$ the bound for the feasible instance yields an anytime regret bound for the tempered action selection rule (\ref{eqn:tempered_selection_rule}) over linear bandit instances.

\begin{lemma}\label{lemma:reg_controlled}
    For any $\delta \in (0,1),$ the following hold for the actions of \textsc{t-eogt} \begin{itemize}
        \item For any feasible instance, \[ \PP( \forall t, \reg_t \le \log(t+t) \cdot \max(U^{\mathsf{F,A}}_t(\log(8/\delta)) , U^{\mathsf{F,B}}_t(\log(8/\delta))) \ge 1-\delta/2.\]
        \item For any infeasible instance, \[\PP( \forall t, \reg_t \le U^{\mathsf{I}}_t( \log(8/\delta)) (1+ \log(t+1)) ) \ge 1-\delta/2. \] 
    \end{itemize} 
    \begin{proof}
        In the feasible case, let $u_t := \max(U^{\mathsf{F,A}}_t(\log(8/\delta)) , U^{\mathsf{F,B}}_t(\log(8/\delta))$. Since $\mathscr{W}_t$ is nondecreasing, and the $U_t^{\mathsf{F, \cdot}}$ are defined as nondecreasing functions of $\mathscr{W}_t,$ it follows that $u_t$ is nondecreasing. By Lemma~\ref{lemma:low_reg_condition}, it follows that \begin{align*} \PP(\exists t : \reg_t > U_t^{\mathsf{F}}\cdot (1 + \log(t+1)) ) &\le \PP(\exists t: \Delta_t \ge u_t/3t, N_t < d/t) + \PP(\exists t: \Delta_t/N_t \ge u_t/3d, N_t \ge d/t). \end{align*}
        But since the events in Lemma~\ref{lemma:delta_separation} must occur with certainty, we have \begin{align*}
            \PP(\exists t: \Delta_t \ge u_t/3t, N_t < d/t) &\le \PP(\exists t : \Delta_t \ge u_t/3t, N_t < d/t, (B_t x_t)^{i_t} \ge \Delta_t/2 - (t/d)^{\nicefrac12}N_t - \sqrt{dN_t})\\
                                                                &\qquad + \PP(\exists t : \Delta_t \ge u_t/3t, N_t < d/t, (B_t x^*)^{i_t^*} \ge \Delta_t/2 + (t/d)^{\nicefrac12}N_t^* + \sqrt{dN_t^*}).
        \end{align*}    
        But, since $u_t \ge U^{\mathsf{F,A}}_t(\log(8/\delta))$, by Lemma~\ref{lemma:poor_events_A}, the first term is at most $\delta/8,$ and similarly since $u_t \ge U^{\mathsf{F,B}}_t( \log(8/\delta0)),$ by Lemma~\ref{lemma:poor_events_B}, the second term is at most $\delta/8,$ controlling the above to $\delta/4$. Of course, the same argument may be repeated to bound $\PP(\exists t: \Delta_t/N_t \ge u_t/3d, N_t \ge d/t),$ giving the first bound. The infeasible case follows the same template, but uses the alternate result in Lemma~\ref{lemma:delta_separation}, and Lemma~\ref{lemma:poor_events_infeasible} to control probabilities instead. We omit the details.
    \end{proof}
\end{lemma}

To concretise the bounds above, we next show an auxiliary lemma controlling the sizes of $U_t^{\mathrm{F,A}}, U_t^{\mathsf{F,B}}$ and $U_t^{\mathsf{I}}.$ \begin{lemma}\label{lemma:explicit_bounds}
    Suppose $\delta \le 1/2.$ Then \begin{align*}
        U_t^{\mathsf{F,A}}(\log(8/\delta))&\le 12d\sqrt{t \log (t+1)} + 15 \sqrt{dt \log(8m/\delta)}\\ 
        U_t^{\mathsf{F,B}}(\log(8/\delta)) &\le 2d^{3/2}\sqrt{t \log^2(t+1)} + 3\sqrt{dt} \log(8m/\delta)\\
        U_t^{\mathsf{I}}( \log(8/\delta)) &\le 6\sqrt{d^2 t (1 + \log(t+1)) + 2dt\log(8m/\delta)}
    \end{align*}
    \begin{proof}
        First, we note that if $\delta \le 1/2,$ then $\frac{1}{2}\log(8/\delta) \ge 3\log(2) > 1$. Thus, we have \begin{align*} \mathscr{W}_t(\log(4/\delta)) &= 1 + \sqrt{\frac d4 \log(1 + t/d) +\frac 12(  \log m + \log(8/\delta))} \\ &\le 2\sqrt{ \frac d4 \log(1 + t) + \frac 12 \log \frac{8m}{\delta} } \\ &\le \sqrt{ d \log(t +1) + 2\log(8m/\delta)}\\ &\le \sqrt{d\log (t+1)} + \frac32 \sqrt{ \log(8m/\delta)}. \end{align*}

        Thus, \[U_t^{\mathsf{F,A}}(\log(8/\delta)) \le 6\sqrt{dt} + 6d\sqrt{t} + 6d \sqrt{t (1 + \log(t+1))} + 9\sqrt{dt \log (8m/\delta)}, \] and further, \[ U_t^{\mathsf{F,B}}(\log(8/\delta)) \le \frac{3 \sqrt{dt}}{2} \cdot \sqrt{ d \log(t +1) + 2\log(8m/\delta)},\] and finally, \[ U_t^{\mathsf{I}}(\log(4/\delta)) \le 6\sqrt{dt} \cdot \sqrt{ d \log(t +1) + 2\log(8m/\delta)},\] yielding the claimed bounds.
    \end{proof}
\end{lemma}

With these in hand, we can conclude.

\begin{proof}[Proof of Lemma~\ref{lemma:tempered_boundary_derivation}]
    We shall only show the feasible case; the infeasible is identical, and thus the details are omitted. Recall from \S\ref{sec:notation} that in the feasible case, \[ \tstat_t \ge t\Gamma - \reg_t + Z_t.\] By Lemma~\ref{lemma:lil_actual}, with probability at least $1-\delta/2, Z_t \ge \lil(t,\delta/2)$ for all $t$. Further, by Lemma~\ref{lemma:reg_controlled}, with probability at least $1-\delta/2,$ at all times \[ \reg_t \le (1 + \log(t+1)) \cdot \max(U^{\mathsf{F,A}}_t(\log(8/\delta)) , U^{\mathsf{F,B}}_t(\log(8/\delta)).\] Finally, opening up the form of the same via , we have \[ (1 + \log(t+1)) \cdot \max\left(12d\sqrt{t \log (t+1)} + 15 \sqrt{dt \log(8m/\delta)}, 2d^{3/2}\sqrt{t \log^2(t+1)} + 3\sqrt{dt} \log(8m/\delta)  \right),\] and $1 + \log(t+1) \le 3 \log(t+1)$ for $t \ge 1$ But note that $d^{3/2} \sqrt{t \log^2(t+1)} \ge d\sqrt{t(1 + \log(t+1))}$, and $\sqrt{dt} \log(8m/\delta) \ge \sqrt{dt \log(8m/\delta)}$ since $\delta \le 1/2.$ So, we may simply adjust the constants, and conclude that with probability at least $1-\delta/2,$ \[ \reg_t \le 36d^{3/2}\sqrt{t} \log^2(t+1) + 45\sqrt{dt \log^2(t+1)}\log(8m/\delta) \le \qfeas_t(\delta) - \lil(t,\delta/2). \] But now the result is obvious.
\end{proof}

\section{Proof of the Lower Bound}\label{appendix:lower_bound}

We conclude the appendix by presenting the proof of the lower bound of Theorem~\ref{theorem:omega_d_lower_bound}. We will first show that it suffices to show a $\Omega(K/\Gamma^2)$ lower bound for the minimum threshold problem (which we shall also formally specify) in order to show the claimed result. We then give a brief summary of the `simulator' technique of \citet{simchowitz2017simulator}, and proceed to show the aforementioned bound.

\subsection{The Finite-Armed Single Objective Feasibility Testing Problem, and a Reduction to Feasibility Testing of LPs over a Simplex}\label{appendix:finite_arm_problem}

We start by explicitly defining the finite-armed single objective feasibility testing problem, also known as the minimum threshold testing problem as discussed in \S\ref{sec:intro}

\paragraph{Problem Definition} An instance of this problem is defined by a natural $K < \infty,$ and a set of $K$ probability distributions, $\{\PP_k\}_{k \in [1:K]},$ each supported over $\mathbb{R},$ and a real $\delta \in (0,1)$. Let $a_k := \mathbb{E}_{S \sim \PP_k}[S]$, and let $a$ denote the $K$-dimensional vector collecting these means. We will assume that $a \in [-1/2, 1/2]^K$. The aim of the test is to distinguish the hypotheses \[ \hfeas^{{K}} : \max_k a_k > 0 \quad \textrm{ versus} \quad \hinfeas^{K}: \max_k a_k < 0.\] The tester chooses an arm $K_t$ in round $t$, and if $K_t = k,$ then it observes in response a score $S \sim \PP_k,$ independently of the history. We shall assume that each $\PP_k$ is $\sigma^2$-subGaussian about its mean, with $\sigma^2 \le 1$. As in the linear setting considered in the main text, a test for this finite-armed single objective setting consists of an arm selection policy, a stopping time, and a decision rule, which we summarise as $\test$ in line with \S\ref{sec:defi}. The goal is reliability in the sense of Definition~\ref{def:reliability}, and a good test should be valid and well adapted in the sense of Definition~\ref{definition:validity_and_adaptedness}.

We now specify reductions of the above problem to the linear feasibility testing problem that is the subject of our paper. The key observation is that the finite-armed problem can either be directly interpreted as a LP feasibility testing problem over a discrete action set, or can, with a small loss in the noise strength, be expressed as a LP feasibility testing problem over a continuous $\mathcal{X}$, the critical implication being that lower bounds for the finite-armed setting extend to our problem of testing feasibility of linear programs. This enables us to only concentrate on showing a lower bound for the finite-armed single objective problem in the subsequent.

\paragraph{Reduction to General LP Feasibility Testing} Note that in effect, the problem above reduces to feasibility testing for the linear case if we set $d = K,$ $A = a^\top \in \mathbb{R}^{1\times d}$ and set $\mathcal{X} = \{ e_i\}_{i = 1}^d$, where the $e_i$ are the standard basis elements for $\mathbb{R}^d$: $e_i= \begin{pmatrix} 0 & \cdots & 0 & 1 & 0 &\cdots 0\end{pmatrix}^\top$, where the $1$ occurs in the $i$th position. Indeed, in this case, upon playing $x = e_i,$ we observe feedback $S \sim \PP_k.$ But we can write $S = \mathbb{E}[S] + (S - \mathbb{E}[S]) = a_k + \zeta = Ax + \zeta,$ where $\zeta = S - \mathbb{E}[S]$ is conditionally $\sigma^2$-subGaussian due to our assumption that each $\PP_k$ is $\sigma^2$-subGaussian, so the reduction is valid if $\sigma^2 \le 1$. 

\paragraph{Reduction to LP Feasiblity Testing Over the Simplex} We further observe that if $\sigma^2 \le 1/2,$ then the finite case also reduces to single constraint feasibility testing over the simplex. Indeed, suppose that we set $d, A$ as above, and take $\mathcal{X} = \{x \in [0,1]^d: \sum x_i = 1 \},$ and let $\test$ be a reliable test for this instance over $1$-subGaussian noise. Then we can get a corresponding reliable test for the $d$-armed setting as follows: 
\begin{itemize}
    \item At each $t$, we first execute $\mathscr{A}$ to obtain a putative action $x_t$.
    \item Next, we \emph{draw} a random index $K_t \sim x_t$, which is meaningful since $x_t$ lies in the simplex, and so is a distribution over $[1:d]$.
    \item Then, we pull arm $K_t$ in the finite-armed instance and we supply the feedback $S_t$ to the linear algorithm to enable testing.
\end{itemize}
To argue that the ensuing test is reliable, we need to verify that the feedback obeys the structure we demand, in particular, that $S_t = Ax_t + \zeta_t$ for $1$-subGaussian $\zeta_t$. But notice that \[ S_t = a_{K_t} + \eta_t\] for $\eta_t$ $\sigma^2$-subGaussian, and further, \[\mathbb{E}[S_t] = \mathbb{E}[a_{K_t}] = \sum x_t^k a_k = a^\top x_t= Ax_t, \] as required. Further, since each $\PP_k$ is supported on $[-1/2, 1/2],$ the random variable $a_{K_t}$ is also supported on $[-1/2, 1/2],$ and so is $1/2$-subGaussian by Hoeffding's inequality. Due to the independence of $K_t$ and $\eta_t$, it follows that the feedback noise is $(1/2 + \sigma^2)$-subGaussian, and so the reduction holds if $\sigma^2 \le 1/2$. 

\paragraph{Improved Costs for Finite Arms.} Prima facie the above reduction implies an $\widetilde{O}(K^2/\Gamma^2)$ stopping cost for our test employed on finite-armed settings. However, if $K < d^2,$ then this may be improved to $\widetilde{O}(K/\Gamma^2)$, either by coupling the \textsc{eogt} approach with direct UCB-based constructions as commonly employed for finite arm bandits, or by directly analysing \textsc{eogt} whilst exploiting standard analyses that enable proofs of improved costs for the OFUL scheme over finite-armed settings \citep{lattimore2020bandit}. 

\subsection{The Simulator Argument}

\newcommand{\BP}{\mathbf{P}}
\newcommand{\simulator}{\mathfrak{S}}
For an execution of a feasibility test over a finite-armed setting, let $N_t^k$ denote the number of times arm $k$ has been pulled up to time $t$, and correspondingly let $N_\tau^k$ be the number of times the arm $k$ has been pulled at stopping. Notice that in a distributional sense, we can view the behaviour of the tester over a fixed \emph{transcript}, defined as a set of $K$ sequences $\{S^k_i\}_{i = 1}^\infty,$ one for each $k$, each comprising of values drawn independently and identically from $\PP_k$, the idea being that for each $t$ such that $K_t = k,$ we can just supply the learner with $S^k_{N_t^k}$ in response. This maintains the feedback distributions, and thus the probability of any event in the filtration induced by $\{\hist_t\}_{t \ge 1}$. The main utility of the transcript view is that it allows manipulation of the distributions underlying an instance \emph{after some number of arm pulls}, and exploiting such distribution shifts is the key insight of the simulator argument of \citet{simchowitz2017simulator}.

Let us succinctly denote a transcript as \( \{ S^k_i \}_{k \in [1:K], i \in [1:\infty)}.\) Further, let us write $\BP = (\PP_1, \cdots, \PP_k)$ to compactly denote an instance, and write $\BP(\cdot)$ to denote the probability of an event when the instance is $\BP$. Throughout, we work with the natural filtration of the tester $\mathscr{F}_t,$ which is the sigma algebra over $\hist_t$ \emph{and} any algorithmic randomness used by the tester. A \emph{simulator} $\simulator$ is a randomised map from transcripts to transcripts. Notice that this induces a new distribution over the behaviour of the algorithm, which we denote by $\BP_{\simulator}$.  Let us say that an event $W \in \mathscr{F}_\tau$ is truthful for an instance $\BP$ under a simulator $\simulator$ if it holds that for every $E \in \mathscr{F}_\tau,$ \[ \BP( W \cap E ) = \BP_{\simulator}(W \cap E).\] In words, given any truthful event, the simulator does not modify the behaviour of the test up to the time it stops. We shall succinctly specify the simulator and distribution with respect to which an event is truthful by saying that `$W$ is $(\BP,\simulator)$-truthful.'

The simulator approach to lower bounds, presented in Proposition 2 of \citet{simchowitz2017simulator}, is summarised through the following bound. Fix an algorithm, and consider a pair of instances $\BP^1$ and $\BP^2$. Then, if $W_1$ is $(\BP^1, \simulator)$-truthful, and $W_2$ is $(\BP^2, \simulator)$-truthful, it holds that \begin{equation}\label{eqn:simulator_lower_bound} \BP^1(W_1^c) + \BP^2(W_2^c) \ge \sup_{E \in \mathscr{F}_\tau} |\BP^1(E) - \BP^2(E)| - \mathrm{TV}(\BP^1_{\simulator}, \BP^2_{\simulator}),\end{equation} where $\mathrm{TV}$ is the total variation distance $\mathrm{TV}(\mu\| \nu) := \sup_{E} \mu(E) - \nu(E)$. The idea thus is that if we construct a simulator that makes the algorithm behave similarly in either instance, i.e., such that \( \mathrm{TV}(\BP^1_{\simulator}\|\BP^2_{\simulator} ) \approx 0,\) but the instances themselves are fundamentally quite different, so that $\sup_{E \in \mathscr{F}_\tau} |\BP^1(E)- \BP^2(E)|$ is large, then we can show lower bounds on how likely truthful events are to not occur.

The bound itself is easy to show: for any $E \in \mathscr{F}_\tau,$ we have \[ |\BP^1(E) - \BP^2(E)| \le |\BP^1_{\simulator}(E) - \BP^2_\simulator(E)| + |\BP^1_{\simulator}(E) - \BP^1(E)| + |\BP^2_\simulator(E) - \BP^2(E)|.\]
Since $W_1$ is $(\BP^1,\simulator)$-truthful, the second term may be refined as \[ |\BP^1_\simulator(E) - \BP^1(E)| = |\BP^1_\simulator(E \cap W_1) - \BP^1(E\cap W_1) + \BP^1_\simulator(E \cap W_1^c) - \BP^1(E \cap W_1^c)| = |\BP^1_\simulator(E \cap W_1^c) - \BP^1(E \cap W_1^c)|, \] and we may similarly bound $|\BP^2_\simulator(E) - \BP^2(E)|$. The difference $|\BP^1_\simulator(E) - \BP^2_{\simulator}(E)|$ can in turn be bounded by the total variation distance. We conclude that \[\sum_{i = 1}^2 \sup_{E \in \mathscr{F}_\tau} |\BP^i_\simulator (E \cap W_i^c) - \BP^i(E \cap W_i^c)| + \mathrm{TV}(\BP^1_\simulator, \BP^2_\simulator) \ge \sup_{E \in \mathscr{F}_\tau} |\BP^1(E) - \BP^2(E)|, \] and the left hand side can be resolved by just taking $E = W_i^c \in \mathscr{F}_\tau$.

We will utilise the above twice in our argument below, with the main trick being that if we only modify the transcript to affect arm $k$ after $T$ pulls, that is, we only change $S^k_i$ for $i > T$, then the event $\{N_\tau^k \le T\}$ is truthful under this simulator, letting us lower bound the probability that $N_\tau^k$ is small in some instance.  We shall succinctly call such simulators \emph{post-$T$ simulators}.

\subsection{A Lower Bound for Finite-Armed Single Constraint Feasibility Testing}

We shall show the following \begin{theorem}\label{theorem:finite_arm_lower_bound}
    For any $\Gamma \in (0,1/2], \delta \le 1/4,$ and $K < \infty,$ and for any reliable test, there exists a finite-armed single objective feasibility testing instance that is feasible, with signal level at least $\Gamma$, $\sigma^2 = 1/2$-subGaussian noise, and $\sum a_k^2 \le 1,$ on which the algorithm must admit \[ \mathbb{E}[\tau] \ge  \frac{(1-2\delta)^3 K}{79 \Gamma^2}. \] 
\end{theorem}

Theorem~\ref{theorem:omega_d_lower_bound} is immediate from the above. \begin{proof}[Proof of Theorem~\ref{theorem:omega_d_lower_bound}] 
    Setting $K = d,$ and constructing either of the reductions from the finite-armed case to the linear program feasibility testing problem detailed in \S\ref{appendix:finite_arm_problem}, which is possible because $\sigma^2 = 1/2$ and since $\|a\|_2 \le 1$. But then the lower bound of Theorem~\ref{theorem:finite_arm_lower_bound} must apply. 
\end{proof}

Without further ado, let us launch into proving the finite-armed lower bound.

\begin{proof}[Proof of Theorem~\ref{theorem:finite_arm_lower_bound}]

    Fix $\Gamma \in (0,1/2]$, and for $k \in [0:K],$ and an $\varepsilon \in (0, \sqrt{1-\Gamma^2/4K}),$ define the following instance \[\BP^k = (\PP_1^k, \cdots, \PP_K^k), \] where \[   \PP_\ell^k = \begin{cases} \mathcal{N}(-\varepsilon, 1/2) & \ell \neq k \\ \mathcal{N}(\Gamma, 1/2) & \ell = k\end{cases}.\] Observe that for $k > 0,$ in instance $\BP^k,$ the $k$th arm is the only feasible action, while the rest are infeasible, while in instance $\BP^0,$ all arms are infeasible, with the tiny signal level $-\varepsilon$. Of course, each $\BP^k$ defines an instance for us. We implicitly reveal to the test that the instance must lie in one of the $\BP^k$ as the argument does not change even if the test is allowed to use this fact. Notice that the mean vector for $\BP^k$ is some permutation of $(\Gamma, -\varepsilon, \cdots, -\varepsilon),$ and so has $2$-norm $\Gamma^2 + (K-1)\varepsilon^2 \le \Gamma^2 + (1-\Gamma^2)/4 \le1,$ since $\Gamma \in (0,1/2].$ We shall, at the end of the proof, send $\varepsilon \to 0$, so the precise size of it is not important to the argument. 

    Now, the first key observation is that since $\BP^k$ is feasible for each $k > 0,$ but $\BP^0$ is infeasible, it must be the case that under $\BP^k$ for arm $k > 0$, the test verifies the feasibility of the instance by pulling arm $k$ at least $\Omega(\Gamma^{-2})$ times. We will need a slightly refined form of this statement, as seen below.    
    \begin{lemma}\label{lemma:arm_with_signal_is_pulled_often}
        Under the above instance structure, for every $k \in [1:K]$ and any $T \in \mathbb{N},$ it holds that \[ \BP^k(N_\tau^k > T) \ge 1-2\delta - \sqrt{T(\Gamma + \varepsilon)^2/2}.\] 
        \begin{proof}
            Consider a post-$T$ simulator $\simulator^k$ such that $ \{\hat{S}^k_i \} = \simulator^k( \{S^k_i\}),$ has the form \[ \hat{S}^{k'}_i = \begin{cases} S^{k'}_i & k'\neq k \textrm{ or } i \le T \\ \overset{\mathrm{i.i.d.}}{\sim} \mathcal{N}(-\varepsilon, 1/2) &  k' = k \textrm{ and } i > T \end{cases}.\] First notice that for the KL divergence\footnote{which we measure in nats, i.e., $\mathrm{KL}(P\|Q) = \int \frac{\mathrm{d}P}{\mathrm{d}Q} \log \frac{\mathrm{d}P}{\mathrm{d}Q}\mathrm{d}Q,$ where the logarithm is natural} \( \mathrm{KL}(\BP^k_{\simulator^k}\|\BP^{0}_{\simulator^k})\), using the data processing inequality, this is bounded by the KL-divergence between the laws of the transcript under the two distrubtions, which in turn is only driven by the the first $T$ entries of the transcript for $T.$ Since, by a standard calculation,\footnote{$\int \frac{e^{- (x-\mu)^2}}{\sqrt{\pi}} ( (x-\nu)^2 - (x-\mu)^2) \mathrm{d}x = \int \frac{e^{- (x-\mu)^2}}{\sqrt{\pi}} (\mu-\nu) (2x-\mu -\nu)\mathrm{d}x = (\mu-\nu) \cdot(2\mu - \mu - \nu)$} \[ \mathrm{KL}(\mathcal{N}(\mu,1/2)\|\mathcal{N}(\nu,1/2)) = (\mu-\nu)^2,\] we conclude that \[ \mathrm{KL}(\BP^k_{\simulator^k}\|\BP^0_{\simulator^k}) \le T (\Gamma + \varepsilon)^2,\] and in turn by an application of Pinsker's inequality\footnote{which says $\mathrm{TV}(P,Q) \le \sqrt{\mathrm{KL}(P\|Q)/2}$.} \citep[see, e.g.,][Chs.13, 14]{lattimore2020bandit}, \[ \mathrm{TV}(\BP^k_{\simulator^k}, \BP^0_{\simulator^k}) \le \sqrt{T(\Gamma + \varepsilon)^2/2}. \]

            Next, observe that the event $W_k := \{ N_\tau^k \le T\}$ is $(\BP^k,\simulator^k)$-truthful since the transcript for arm $i$ is only modified after $T$ pulls, and further, every event is $(\BP^0, \simulator^k)$-truthful since the simulator does not modify the arm distributions for $\BP^0,$ and so in particular $W_0 = \{N_\tau^k \le \infty\}$ is truthful (and of course $\BP^0(N_\tau^k > \infty) = 0$ trivially). 

            Finally, observe that since the instance $\BP^k$ is feasible, and since the test is reliable, it holds that $\BP^k(\decrule(\hist_\tau) = \hfeas^K) \ge 1-\delta.$ But by the same coin, since $\BP^0$ is infeasible, $\BP^0(\decrule(\hist_\tau) = \hinfeas^K) \le \delta$. Of course, $\{\decrule(\hist_\tau) = \hfeas^K\}$ is an $\mathscr{F}_\tau$-event. 

            So, we may proceed to populate the inequality (\ref{eqn:simulator_lower_bound}) with the above selections to conclude that \[ \BP^k(N_\tau^K > T) + 0 \ge |1-\delta - \delta| - \sqrt{T(\Gamma + \varepsilon)^2/2}.\qedhere\]
        \end{proof}
    \end{lemma}

    With the above in hand, observe that since $\tau = \sum_{k = 1}^K N^k_\tau,$ the above already shows that $\mathbb{E}_{\BP^k}[\tau] = \Omega(\Gamma^{-2})$. To extend this, we employ the following result.
    
    \begin{lemma}\label{lemma:weak_arms_are_pulled_often}
        Under the same setting as Lemma~\ref{lemma:arm_with_signal_is_pulled_often}, for any $k, k' \in [1:K],$ \[ \BP^k(N_\tau^{k'} > T) + \BP^{k'}(N_\tau^k > T) \ge \frac{1-2\delta}{2} - \frac{1 + 1/\sqrt{2}}{2} \sqrt{ T (\Gamma + \varepsilon)^2}. \]
        \begin{proof}
            If $k = k',$ the claim is true due to Lemma~\ref{lemma:arm_with_signal_is_pulled_often}. Without loss of generality, let us set $k = 1, k' = 2$. Define the simulator $\simulator^{1\to 2}$ so that $\{\hat{S}^k_i\} = \simulator^{1\to 2}(\{S^k_i\})$ has the form \[ \hat{S}^k_i = \begin{cases}
                S^k_i & k \not\in \{1,2\} \textrm{ or } i \le T \\ \overset{\mathrm{i.i.d.}}{\sim} \mathcal{N}(\Gamma, 1/2) & k \in \{1,2\} \textrm{ and } i > T
            \end{cases}. \]

            As in the proof of Lemma~\ref{lemma:arm_with_signal_is_pulled_often}, the only difference between $\BP^1_{\simulator^{1\to 2}}$ and $\BP^2_{\simulator^{1\to 2}}$ is induced by the first $T$ entries of the $k = 1$ and $k = 2$ rows, and thus \[\mathrm{KL}(\BP^1_{\simulator^{1\to 2}}\|\BP^2_{\simulator^{1\to 2}}) \le 2\cdot T(\Gamma + \varepsilon)^2.\]

            Further, again, $W_1 := \{N_\tau^2 \le T\}$ is $(\BP^1, \simulator^{1\to 2})$-truthful, since for $\BP^1,$ the simulator $\simulator^{1\to 2}$ only modifies the the law of arm $2,$ and does this only after $T$ pulls of the same. Similarly, $W_2 := \{N_\tau^1 \le T\}$ is $(\BP^2,\simulator^{1 \to 2})$-truthful. Now set $E = \{ N_\tau^2 >  T\}$. Then by (\ref{eqn:simulator_lower_bound}), we have \[ \BP^1(N_\tau^2 > T) + \BP^2(N_\tau^1 > T) \ge |\BP^1(N_\tau^2 > T) - \BP^2(N_\tau^2 > T)| - \sqrt{ T(\Gamma + \varepsilon)^2}.\]

            Now observe that if $\BP^1(N_\tau^2 > T) \ge \BP^2(N_\tau^2 > T),$ then we are already done since by Lemma~\ref{lemma:arm_with_signal_is_pulled_often},\[ \BP^2(N_\tau^2 > T) \ge 1-2\delta - \sqrt{T(\Gamma + \varepsilon)^2/2} > \frac{1 - 2\delta}{2} - \frac{1 + 1/\sqrt{2}}{2}\sqrt{T (\Gamma + \varepsilon)^2}.\] So, we may assume that $\BP^1(N_\tau^2 > T) \le \BP^2(N_\tau^2 > T)$. But then we conclude that \[ 2\BP^1(N_\tau^2 > T) + \BP^2(N_\tau^1 > T) \ge \BP^2(N_\tau^2 > T) - \sqrt{ T(\Gamma + \varepsilon)^2} \ge 1-2\delta - (1+1/\sqrt{2})\sqrt{T (\Gamma + \varepsilon)^2},\] and the conclusion follows since $2\BP^1(N_\tau^2 > T) + \BP^2(N_\tau^1 > T) \le 2\BP^1(N_\tau^2 > T) + 2\BP^2(N_\tau^1 > T)$.                       
        \end{proof}
    \end{lemma}
    
    With the above in hand, observe that since an arm is pulled at each $t, \tau = \sum_k N_\tau^k$. Thus, for any $T > 0,$ \begin{align*}
        \frac{1}{K} \sum_k \mathbb{E}_{\BP^k}[\tau] &= \frac{1}{K} \sum_k \sum_{k'} \mathbb{E}_{\BP^{k}}[N_\tau^{k'}] \\
                                                    &\ge \frac{1}{K} \sum_k \sum_{k'} T \BP^{k'}(N_\tau^k > T)\\
                                                    &= \frac{T}{K} \left( \sum_k \BP^k (N_\tau^k > T) +  \frac12 \sum_{k, k' \neq k} \BP^k (N_\tau^{k'} > T) + \BP^{k'}(N_\tau^{k}> T) \right). 
    \end{align*}

    Now employing Lemma~\ref{lemma:arm_with_signal_is_pulled_often} and Lemma~\ref{lemma:weak_arms_are_pulled_often}, we have \begin{align*}
        \frac{1}{K} \sum_k \mathbb{E}_{\BP^k}[\tau] &\ge \frac{T}{K} \left(1 - 2\delta - \sqrt{ T(\Gamma+ \varepsilon)^2/4} \right) + \frac{TK(K-1)}{2K} \left( \frac{1-2\delta}{2} - \frac{1 + 1/\sqrt{2}}{2} \sqrt{T(\Gamma + \varepsilon)^2}\right)\\
            &\ge \frac{TK}{4}\left( (1-2\delta) -  (1 + 1/\sqrt{2})\sqrt{(T (\Gamma + \varepsilon)^2)}.\right)
    \end{align*} 

    Since the bound holds for every $T$, we can optimise the same\footnote{For $f(x) = u x - vx^{3/2},$ the derivative is $u - \frac{3v}{2}\sqrt{x},$ while the second derivative is negative over $[0,\infty),$ yielding the global maxima at $(2u/3v)^2,$ with the maximum evaluating to $4u^3/9v^2 - 8u^3/27v^2 = \frac{4u^3}{27v^2}.$ Setting $u = (1-2\delta), v = (1 + 1/\sqrt{2})(\Gamma + \varepsilon),$ this evaluates to $\frac{4}{27} \cdot \frac{(1-2\delta)^3}{(1 +1/\sqrt{2})^2 (\Gamma + \varepsilon)^2}.$  } to conclude that \[ \max_k \mathbb{E}_{\BP^k}[\tau] \ge \frac{1}{K} \sum_k \mathbb{E}_{\BP^k}[\tau] \ge \frac{(1-2\delta)^3}{27 (1 + 1/2 + \sqrt{2})} \cdot \frac{(1-2\delta)^3 K}{(\Gamma + \varepsilon)^2} \ge \frac{(1-2\delta)^3K}{79(\Gamma + \varepsilon)^2}.\]

    If $\delta \le \frac14,$ this can be further lower bounded by $\frac{K}{632(\Gamma + \varepsilon)^2}$. Since the above inequality holds true for every $\varepsilon > 0$ small enough, the claimed result follows upon sending $\varepsilon \to 0$.    
\end{proof}

\end{document}